\documentclass[11pt]{article}
\usepackage[utf8]{inputenc}
\usepackage{style}
\usepackage[left=1in,bottom=1in,top=1in,right=1in]{geometry}

\usepackage{hyperref, cleveref}

\hypersetup{%
	colorlinks   = true, 
	urlcolor     = blue, 
	linkcolor    = blue, 
	citecolor    = blue,  
}

\title{A Characterization of List Learnability}
\author{Moses Charikar\thanks{Stanford University. Email: \texttt{moses@cs.stanford.edu.}}
\and
Chirag Pabbaraju\thanks{Stanford University. Email: \texttt{cpabbara@cs.stanford.edu.}}}

\date{\today}

\begin{document}

\maketitle

\begin{abstract}
A classical result in learning theory shows the equivalence of PAC learnability of binary hypothesis classes and the finiteness of VC dimension.
Extending this to the multiclass setting was an open problem, which was settled in a recent breakthrough result characterizing multiclass PAC learnability via the DS dimension introduced earlier by Daniely and Shalev-Shwartz.

In this work we consider list PAC learning where the goal is to output a list of $k$ predictions. List learning algorithms have been developed in several settings before and indeed, list learning played an important role in the recent characterization of multiclass learnability.
In this work we ask: when is it possible to $k$-list learn a hypothesis class?

We completely characterize $k$-list learnability in terms of a generalization of DS dimension that we call the $k$-DS dimension.
Generalizing the recent characterization of multiclass learnability, 
we show that a hypothesis class is $k$-list learnable if and only if the $k$-DS dimension is finite.

\end{abstract}
\newpage
\section{Introduction}
\label{sec:intro}

An important direction in statistical learning theory is the study of learning hypothesis classes in the \textit{Probably Approximately Correct (PAC)} framework introduced by \cite{valiant1984theory}. 
When the number of classes is 2, a fundamental result shows that the \textit{Vapnik-Chervonenkis (VC) dimension} \cite{vapnik1974theory, vapnik2015uniform, blumer1989learnability} of the hypothesis class \textit{completely characterizes} learning (in both, realizable and agnostic settings), 
i.e. 
(1) finiteness of the dimension ensures the existence of a learning algorithm which successfully PAC learns the class (sufficiency), and conversely, (2) no hypothesis class for which the dimension is infinite can be PAC learned by any algorithm (necessity).

However, in practice, most classification problems (e.g. object recognition, speech recognition, document categorization) involve the multiclass setting, with more than 2 classes. 
The question 
of what dimension characterizes learnability in the multiclass setting
was a long-standing open problem in the learning theory literature. 
The \textit{Natarajan dimension} was proposed by \cite{natarajan1989learning, natarajan1988two} and was shown to completely characterize learning when the number of labels is finite.
When the number of labels is infinite, 
the Natarajan dimension satisfies the necessity condition (2), but whether the Natarajan dimension also satisfies sufficiency (1) remained open. 

In a recent major breakthrough, \cite{brukhim2022characterization} 
resolved this question
--- they constructed a hypothesis class that has Natarajan dimension 1, but is not PAC learnable. Instead, they proved that another dimension, called the \textit{Daniely-Shalev-Shwartz (DS) dimension}, completely characterizes multiclass learnability. The DS dimension was introduced in the work of \cite{daniely2014optimal}, where they had shown that it satisfies the necessity condition (2). Using beautiful machinery that involves a wide assortment of techniques, \cite{brukhim2022characterization} showed the existence of an algorithm that successfully learns any hypothesis class that has finite DS dimension, \textit{even when the number of labels is infinite}. 

How can we hope to learn hypothesis classes with infinite DS dimension?
Since standard PAC learning is not possible, 
predicting a short \textit{list} of labels
is a natural alternative. 
The focus of our work is on understanding when list learning is possible.
We obtain an exact characterization of when we can list learn a hypothesis class, where the algorithm is allowed to output a list of $k$ labels.

In fact, list learning is an important ingredient in the construction of the learning algorithm in \cite{brukhim2022characterization} (although, the notion of list learning that we care about in this work is slightly different, see Remarks \ref{rem:list-learning-definition-in-BCD+22},\ref{rem:weak-list-learner-vs-strong}).\footnote{Furthermore, the list learning algorithm described in \cite[Algorithm 2]{brukhim2022characterization} has sample complexity that scales with the DS dimension, and hence is not applicable in the infinite DS dimension regime.} One can think of several settings where one might want to predict a short list of labels, instead of just a single prediction. For example, in recommendation systems, the goal is to be able to accurately predict and recommend a short menu of items to a customer, so that they are likely to choose/purchase one of the items in the menu with good chance. This can be applicable in the context of online shopping websites like Amazon, eBay etc., advertisements on social media, music recommendations on YouTube, Spotify, etc. Even if the ultimate goal is to make a single prediction, it is useful to construct a short list from a very large domain of possible labels. 
For example, in the context of medicine, it might be very helpful to a doctor if an algorithm, when provided symptoms of a patient as input, outputs a short list of possible 
health issues causing the patient's current condition.
In the context of hiring for a single position in companies, one may want to shortlist candidates based on their resumes in an automated manner. In these and other cases, where the final decision may be needed to be taken by a human expert, list learning significantly simplifies the task at hand. 

In a separate context, one can also think of robust learning settings where a fraction of the available training data is adversarially corrupted.  Can one still make accurate predictions on test points? It is unreasonable to ask for single predictions that are correct with high probability in such settings, especially when a large fraction of the data may be arbitrarily corrupted. However, if one is granted the ability to predict lists, it turns out that meaningful predictions become possible, in what is known as the \textit{list-decodable learning} framework \cite{balcan2008discriminative}. In this framework, only an $\alpha$ fraction of the training data is uncorrupted, and there are no assumptions made on the remaining $1-\alpha$ fraction --- the goal is to predict short lists such that the true answer for the task at hand is close to one of the members of the predicted list. 
\cite{charikar2017learning} show that for the robust mean estimation problem (among other problems), if the true data is drawn from a distribution with bounded covariance, one can predict lists of size $1/\poly(\alpha)$ and achieve an error rate of $\widetilde{O}(1/\sqrt{\alpha})$. 
Thereafter, many subsequent works \cite{diakonikolas2018list, diakonikolas2020list, karmalkar2019list, raghavendra2020list} have attempted to improve the running time of their algorithm and/or achieve better error rates under stronger assumptions;
for an exhaustive survey of advances in this field of robust statistics, we refer the reader to \cite{diakonikolas2019recent}. 
Lastly, given the extent to which list-decoding has empowered the rich field of coding theory, it is also natural to wonder if there are connections (\cite{guruswami1999multiclass} is an excellent example) to the techniques used there to devise good list learning algorithms.


\section{Overview of results}
\label{sec:overview}

In this work, we investigate the following questions about list learning:
\begin{enumerate}
\item \ul{\textit{Which hypothesis classes can be list-learned for a given list size $k$, even when their $DS$ dimension is infinite?}}
\item \underline{\textit{What is a successful list learning algorithm for such hypothesis classes?}}
\end{enumerate}
Our main technical contribution is to completely characterize \textit{when} and \textit{how} list learning (for a fixed target list size $k$) can be achieved. Towards this, we introduce a complexity parameter for a hypothesis class, which we denote as the \textit{$k$-Daniely-Shalev-Shwartz ($k$-DS) dimension}, and show that it completely governs $k$-list learnability. The $k$-DS dimension is a natural analogue of the DS dimension \cite{daniely2014optimal} of a hypothesis class, which was shown to completely characterize multiclass learnability in \cite{brukhim2022characterization}. We build on their work, to 
obtain the following result:
\begin{align*}
    \underline{\textit{$k$-list learnability $\Leftrightarrow$ finite $k$-DS dimension.}}
\end{align*}
The necessity of finiteness for learnability is implied by the following theorem:
\begin{theorem}[Finite $k$-DS dimension is necessary]
    \label{thm:necessity}
    If a hypothesis class has infinite $k$-DS dimension, it is not $k$-list PAC learnable.
\end{theorem}
Sufficiency (both, in the realizable and agnostic learning settings) is implied by the following theorem:
\begin{theorem}[Finite $k$-DS dimension is sufficient]
    \label{thm:sufficiency}
    Let $\mcH \subseteq \mcY^\mcX$ be a hypothesis class with $k$-DS dimension $\dds < \infty$. Then, there exist list learning algorithms $\mcA^{real}$ and $\mcA^{agn}$ with the following guarantees:\footnote{$\widetilde{O}(\cdot)$ hides $\mathrm{polylog}(m, \dds, k)$ factors. All logarithms in this paper are in the natural base $e$.}
    \begin{enumerate}
        \item For every distribution $\mcD$ on $(\mcX \times \mcY)$ realizable by $\mcH$, every $\delta \in (0,1)$, and every integer $m \in \N$, with probability at least $1-\delta$ over a sample $S$ of size $m$ drawn i.i.d. from $\mcD$, the $k$-list hypothesis $\mu^k=\mcA^{real}(S)$ satisfies
        \begin{equation*}
            \Pr_{(x,y)\sim\mcD}\left[y \notin \mu^k(x)\right] \le \widetilde{O}\left(\frac{k^6(\dds)^{1.5}+\log(1/\delta)}{m}\right).
        \end{equation*}
        \item For every distribution $\mcD$ on $(\mcX \times \mcY)$, every $\delta \in (0,1)$, and every integer $m \in \N$, with probability at least $1-\delta$ over a sample $S$ of size $m$ drawn i.i.d. from $\mcD$, the $k$-list hypothesis $\mu^k=\mcA^{agn}(S)$ satisfies
        \begin{equation*}
            \Pr_{(x,y)\sim\mcD}\left[y \notin \mu^k(x)\right] \le \inf_{h \in \mcH} L_{\mcD}(h) + \widetilde{O}\left(\sqrt{\frac{k^6(\dds)^{1.5}+\log(1/\delta)}{m}}\right),
        \end{equation*}
        where $L_{\mcD}(h) = \Pr_{(x,y) \sim \mcD}[y \neq h(x)]$.
    \end{enumerate}
\end{theorem}
Taken together, the two theorems above provide a complete picture about list PAC learnability for a fixed list size $k$.
It is remarkable that we are able to obtain such a tight characterization in terms of the list size $k$ -- in particular, note that we do not need to relax the list size $k$ in the algorithms guaranteed by the sufficiency statement.
The dependence on $k$-DS dimension above is qualitatively the same as that on DS dimension in \cite[Theorem 1]{brukhim2022characterization}, but our bounds suffer an additional small polynomial factor in the list size. In the case when the DS dimension of the class is infinite (so that we cannot obtain PAC learning guarantees for it), if the $k$-DS dimension of the class is finite for some $k$ (see \Cref{eg-finite-ds-infinite-k-ds}), \Cref{thm:sufficiency} 
shows that the class can be list learned.
The rest of the paper is dedicated towards proving these theorems, and is organized as follows:

\paragraph{Roadmap.}  First, in \Cref{sec:preliminaries}, we set up notation and formalize the context of our work. We provide definitions of various dimensions (Natarajan, DS, exponential) from the literature that are relevant in the discussion of PAC learnability. We also define the notion of \textit{one-inclusion graphs} of hypothesis classes \cite{haussler1994predicting, rubinstein2006shifting}, and discuss the task of finding good \textit{orientations} of these graphs. The challenge is to pick the correct generalizations of all these objects so that they fit together to give us a list learning algorithm. We devise list analogues for each of these dimensions ($k$-Natarajan, $k$-DS, $k$-exponential) as well as quantities related to the one-inclusion graph ($k$-degree, list orientations, etc.). We also demonstrate the power of list PAC learning by providing an example of a hypothesis class that has \textit{infinite} DS dimension (so that it is not PAC learnable), but finite $k$-DS dimension (so that, by \Cref{thm:sufficiency} above, it is $k$-list PAC-learnable). 

Then, in \Cref{sec:necessity}, we prove \Cref{thm:necessity} above. Inspired by the lower bound construction in \cite{daniely2014optimal}, we instantiate a hard learning problem that directly captures the crux behind the definition of the $k$-DS dimension. By showing that any list-learning algorithm errs with a good chance on this hard instance, we show that no algorithm can achieve arbitrarily low error \textit{for any sample size} when the $k$-DS dimension of the hypothesis class is infinite.

Next, 
in a sequence of steps, we 
construct a successful list learning algorithm for hypothesis classes having finite $k$-DS dimension. In \Cref{sec:list-sauer-lemma}, we revisit a classical result in learning theory, namely Sauer's lemma \cite{sauer1972density, shelah1972combinatorial} (also known as the Perles–Sauer–Shelah lemma), which controls the growth of the size of a hypothesis class when restricted to sets of increasing size. Sauer's lemma is an important stepping stone in proving learnability results in both binary and multiclass learning \cite{haussler1995generalization}. We successfully show a list-variant of Sauer's lemma, where we are interested in controlling the growth of a hypothesis class in terms of its relevant \textit{list} dimensions, when restricted to sets of increasing size. It is worth mentioning that a result of this form was used in an entirely different context -- inapproximability in deterministic truthful mechanisms -- in the work of \cite{daniely2015inapproximability}. In our work, we need the list-variant of Sauer's lemma as a crucial ingredient to prove our sufficiency theorem.

Thereafter, in \Cref{sec:shifting}, we 
generalize the beautiful combinatorial technique of shifting used in \cite{brukhim2022characterization}. At a high level, an object called the \textit{one-inclusion graph} of a hypothesis class (see \Cref{def:oig}) turns out to be very important in devising learning algorithms. In particular, one is interested in \textit{orienting the edges} of this graph in a smart way --- good orientations result in accurate learning algorithms. In the pursuit of coming up with nice orientations of the one-inclusion graph, the technique of shifting comes in extremely handy. Shifting simplifies a hypothesis class to a class that has related list dimensions, but is easier to reason about if we are interested in orientations. 
In \Cref{sec:shifting}, we adapt
the results in \cite{brukhim2022characterization} in a step-by-step manner to the list learning setting. We carefully track how the relevant list dimensions change during the process of shifting, before finally coming up with a greedy method for constructing good orientations.

Armed with the list-version of Sauer's lemma and the ability to construct good orientations of one-inclusion graphs, we combine these tools in \Cref{sec:sufficiency}. This section is also heavily inspired by the proof methodology in \cite{brukhim2022characterization}. We first establish a list version of the one-inclusion learning algorithm \cite{haussler1994predicting, rubinstein2006shifting} as a major building block, and state some nice properties about it. Then, we show that every hypothesis class with finite $k$-DS dimension has a 
\textit{list sample compression scheme}. We are interested in compression because existing results in the literature \cite{david2016statistical} equate learnability with compressibility, and these results are directly amenable to the list learning setting. Intuitively, compressability means that given a sample, it is possible to produce a smaller sub-sample from which the labels on the sample can be recovered --- we show that finite $k$-DS dimension gives a hypothesis class this ability, and consequently makes it list-learnable, thus completing the proof of \Cref{thm:sufficiency}.

\subsection{Other related work}
\label{sec:related-work}
\paragraph{Beyond binary classes} The quest for a dimension capturing PAC learnability in the multiclass setting has indeed been challenging. One of the significant obstacles that arises in the multiclass setting with infinite labels is that the \textit{Empirical Risk Minimization} (ERM) rule ceases to be a learner \cite{daniely2014optimal, daniely2015multiclass}, and as a result, uniform convergence results from the binary case do not translate over. When the number of labels is bounded, \cite{ben1992characterizations} identified a general framework of dimensions called \textit{$\Psi$-distinguishers} that characterize PAC learnability. The Natarajan dimension, Graph dimension \cite{natarajan1989learning}, as well as Pollard's pseudo-dimension \cite{pollard1990empirical} are all special cases within this framework. However, this framework no longer applies when the number of labels is infinite.

\paragraph{Beyond PAC learning}
Recently, there has also been a flurry of fascinating work that seeks to go beyond the classical PAC learning model. In practical machine learning settings, the source of data is generally fixed; however, since PAC learning is concerned only with worst-case error rates that hold uniformly over all data distributions, it can be an overly pessimistic model. To remedy this, \cite{bousquet2021theory} initiated the study of distribution-dependent error rates in the model of \textit{universal learning}. In this model, they show a surprising phenomenon that any hypothesis class can only fall in \textit{one of three} error rate regimes --- linear, exponential or arbitrarily slow rates.

Another recent work is that by \cite{alon2022theory} which seeks to model known assumptions on the data via \textit{partial} concept classes instead of \textit{total} concept classes. While total (binary) concept classes necessarily need to map the entire data domain to $\{0,1\}$, partial classes are allowed to be undefined $(\star)$ on some part of the domain. The regions where they are undefined precisely capture the underlying assumptions on the data (e.g. if it is known that the data is linearly separated by a margin, the partial concept can map the margin area to $\star$).

Most recently, \cite{kalavasis2022multiclass} extend both the above-mentioned frameworks of universal learning and partial classes to the multiclass setting (although they restrict themselves to bounded labels as well). Interestingly, the one-inclusion graph prediction strategy \cite{haussler1994predicting, rubinstein2006shifting} plays an important role in all these works, as it also does in \cite{brukhim2022characterization}.

Finally, \cite{cheraghchi2021list} also consider list learning in the \textit{attribute-noise model} \cite{shackelford1988learning}. In this model, the attributes of the input data are allowed to be perturbed by noise, and the goal is to still be able to obtain PAC-like guarantees with respect to clean data. \cite{cheraghchi2021list} consider the setting of learning binary hypotheses in this model, and show that one can circumvent (in certain cases) otherwise information-theoretically impossible results by outputting a short list of binary hypotheses. Our object of focus in this paper is slightly different, in that we are still in the noise-free setting, but want to precisely characterize (by a suitable complexity parameter) when a class (with potentially infinite labels) can be list-learned (in the PAC sense, as defined in \Cref{def:list-pac-learnability} ahead).
\section{Preliminaries and Definitions}
\label{sec:preliminaries}

\paragraph{Notation.}
The data domain is denoted as $\mcX$ and the label domain is denoted as $\mcY$. A hypothesis class $\mcH$ is a subset $\mcH \subseteq \mcY^\mcX$. A distribution $\mcD$ over $(\mcX \times \mcY)$ is said to realizable by $\mcH$ if there exists $h^* \in \mcH$ such that $\Pr_{(x,y) \sim \mcD}[h^*(x) \neq y] = 0$. For a sequence $S \in \mcX^m$, we refer to the restriction of $\mcH$ to $S$ as $\mcH|_S$, and we think of $\mcH|_S$ as a set of vectors of size $m$. $\N$ denotes the set of positive integers. Often times, $\mcX$ may be $[m]$ and $\mcY$ may be $[p]^m$ for $p,m \in \N$. Since we are interested in the problem of predicting short lists of size $k$ on a learning problem with a large number of labels, we will generally operate in a regime where $k \ll |\mcY|$. We routinely refer to lists of size $k$ as $k$-lists.

\subsection{Dimensions and learnability}
\label{sec:dimensions-learnability}
We first recall the definition of PAC learning, due to \cite{valiant1984theory}. In this framework, one receives a training sample drawn i.i.d. from some arbitrary distribution on the data and labeled by some target hypothesis in a hypothesis class. This is known as the \textit{realizable} case, and the goal of any learning algorithm is to output a hypothesis, which with good chance, agrees with the target hypothesis on a new point sampled from the same distribution. In the \textit{agnostic} case, one receives an arbitrary sample from any \textit{joint distribution} over the data and labels (so that no hypothesis in the class may completely agree with the sample), but the goal then is to output a hypothesis that achieves low error relative to the \textit{best hypothesis} in the class on a new test point drawn from the same distribution. We restate the formal definition (for the realizable setting) below, and refer the reader to \cite[Chapter 3]{shalev2014understanding} for a more complete exposition on realizable as well as agnostic PAC learning.
\begin{definition}[PAC learning \cite{valiant1984theory}]
    \label{def:pac-learnability}
    Let $\mcH \subseteq \mcY^{\mcX}$ be a hypothesis class. Let $\mcD$ be any distribution over $(\mcX \times \mcY)$ realizable by $\mcH$. For a hypothesis $h:\mcX \to \mcY$, define $\err_{\mcD}(h)=\Pr_{(x,y) \sim \mcD}[h(x)\neq y]$. We say that $\mcH$ is PAC learnable by a learning algorithm $\mcA$ with sample complexity $m_{\mcA, \mcH}: (0,1) \times (0,1)\to\N$ if for every $\epsilon, \delta > 0$, every distribution $\mcD$ realizable by $\mcH$, for $m \ge m_{\mcA, \mcH}(\epsilon, \delta)$, $\Pr_{S \sim \mcD^m}[\err_{\mcD}(\mcA(S)) \ge \epsilon]\le\delta$, where $\mcA(S)$ is the hypothesis output by the learning algorithm on input $S$.
\end{definition}

The definition above naturally lends itself to the notion of \textit{list PAC learning} that we care about:
\begin{definition}[List PAC learning]
    \label{def:list-pac-learnability}
    Let $\mcH \subseteq \mcY^{\mcX}$ be a hypothesis class. Let $\mcD$ be any distribution over $(\mcX \times \mcY)$ realizable by $\mcH$. For a $k$-list hypothesis $\mu^k:\mcX \to \{Y \subseteq \mcY:|Y|\le k\}$, define $\err_{\mcD}(\mu^k)=\Pr_{(x,y) \sim \mcD}[\mu^k(x)\not\owns y]$. We say that $\mcH$ is $k$-list PAC learnable by a list learning algorithm $\mcA$ with sample complexity $m^k_{\mcA, \mcH}: (0,1) \times (0,1)\to\N$ if for every $\epsilon, \delta > 0$, every distribution $\mcD$ realizable by $\mcH$, for $m \ge m^k_{\mcA, \mcH}(\epsilon, \delta)$, $\Pr_{S \sim \mcD^m}[\err_{\mcD}(\mcA(S)) \ge \epsilon]\le\delta$, where $\mcA(S)$ is the $k$-list hypothesis output by the list learning algorithm on input $S$.
\end{definition}

\begin{remark}
    \label{rem:list-learning-definition-in-BCD+22}
    We note here that a similar definition of list PAC learnability is given in \cite[Definition 30]{brukhim2022characterization} --- however, the important thing is that the list size $k$ for the purposes of their definition can depend on the success probability. In particular, the list size can be quite large for a high success probability. In contrast, the definition above requires list learnability for a \textit{fixed} list size \textit{irrespective} of the success probability.
\end{remark}

Next, we go on to define several notions of ``dimensions'' for hypothesis classes, each of which captures an appropriate notion of complexity of the hypothesis class. Our first definition is that of the DS dimension.

\begin{definition}[DS dimension \cite{daniely2014optimal}]
    \label{def:ds-dim}
    Let $\mcH \subseteq \mcY^\mcX$ be a hypothesis class and let $S \in \mcX^d$ be a sequence. Let us think of the members of $\mcH|_S$ as vectors in $\mcY^d$. For $i \in [d]$, we say that $f,g \in \mcH|_S$ are $i$-neighbours if $f_i \neq g_i$ and $f_j = g_j, \; \forall j \neq i$. We say that $\mcH$ DS shatters $S$ if there exists $\mcF \subseteq \mcH, |\mcF|<\infty$ such that $\forall f \in \mcF|_S ,\; \forall i \in [d],$ $f$ has at least one $i$-neighbor. The DS dimension of $\mcH$, denoted as $d_{DS}=d_{DS}(\mcH)$, is the largest integer $d$ such that $\mcH$ DS shatters some sequence $S \in \mcX^d$.
\end{definition}
While \cite{daniely2014optimal} showed that the DS dimension of a hypothesis class needs to be finite for it to be multi-class learnable, \cite{brukhim2022characterization} showed that finite DS dimension is in fact sufficient for learnability. A natural analogue of this dimension in the context of predicting lists of size $k$ is as follows:

\begin{definition}[$k$-DS dimension]
    \label{def:k-ds-dim}
    Let $\mcH \subseteq \mcY^\mcX$ be a hypothesis class and let $S \in \mcX^d$ be a sequence.
    We say that $\mcH$ $k$-DS shatters $S$ if there exists $\mcF \subseteq \mcH, |\mcF|<\infty$ such that $\forall f \in \mcF|_S ,\; \forall i \in [d],$ $f$ has at least $k$ $i$-neighbors. The $k$-DS dimension of $\mcH$, denoted as $\dds=\dds(\mcH)$, is the largest integer $d$ such that $\mcH$ $k$-DS shatters some sequence $S \in \mcX^d$.
\end{definition}

Observe that the two definitions above differ only in the number of $i$-neighbors they require each hypothesis to have --- it is easy to see that $d^k_{DS} > d^{k'}_{DS}$ for $k < k'$. In fact, we can easily construct a hypothesis class that has infinite DS dimension, but finite $k$-DS dimension for $k \ge 2$.

\begin{example}
\label{eg-finite-ds-infinite-k-ds}
Consider the hypothesis class $\mcH \subseteq \N^{\N}$ defined as below:

\begin{alignat*}{10}
\mcX\quad=\qquad &\quad1\qquad &&\quad2\qquad &&\quad3\qquad &&\quad4\qquad &&\quad5\qquad &&\quad6\qquad &&\quad7\qquad &&\quad8\qquad &&\quad9\qquad &&\;\dots\qquad \\
\hline\\
&\begin{pmatrix}
1 \\ 2 \\ 3 
\end{pmatrix}\times
&&\begin{pmatrix}
1 \\ 2 \\ 3
\end{pmatrix}\times
&&\begin{pmatrix}
1 \\ 2 \\ 3
\end{pmatrix}\times
&&\begin{pmatrix}
1 \\ 2 
\end{pmatrix}\times
&&\begin{pmatrix}
1 \\ 2 
\end{pmatrix}\times
&&\begin{pmatrix}
1 \\ 2 
\end{pmatrix}\times&&\begin{pmatrix}
1 \\ 2 
\end{pmatrix}\times
&&\begin{pmatrix}
1 \\ 2 
\end{pmatrix}\times
&&\begin{pmatrix}
1 \\ 2 
\end{pmatrix}\times
&&\;\dots
\\
& && && && && \;\bigcup && && && && && \\
\mcH\quad=\qquad&\begin{pmatrix}
4 \\ 5 \\ 6
\end{pmatrix}\times
&&\begin{pmatrix}
4 \\ 5 \\ 6
\end{pmatrix}\times
&&\begin{pmatrix}
4 \\ 5 \\ 6
\end{pmatrix}\times
&&\begin{pmatrix}
3 \\ 4
\end{pmatrix}\times
&&\begin{pmatrix}
3 \\ 4
\end{pmatrix}\times
&&\begin{pmatrix}
3 \\ 4
\end{pmatrix}\times&&\begin{pmatrix}
3 \\ 4
\end{pmatrix}\times
&&\begin{pmatrix}
3 \\ 4
\end{pmatrix}\times
&&\begin{pmatrix}
3 \\ 4
\end{pmatrix}\times
&&\;\dots
\\
& && && && && \;\bigcup && && && && && \\
&\begin{pmatrix}
7 \\ 8 \\ 9
\end{pmatrix}\times
&&\begin{pmatrix}
7 \\ 8 \\ 9
\end{pmatrix}\times
&&\begin{pmatrix}
7 \\ 8 \\ 9
\end{pmatrix}\times
&&\begin{pmatrix}
5 \\ 6
\end{pmatrix}\times
&&\begin{pmatrix}
5 \\ 6
\end{pmatrix}\times
&&\begin{pmatrix}
5 \\ 6
\end{pmatrix}\times&&\begin{pmatrix}
5 \\ 6
\end{pmatrix}\times
&&\begin{pmatrix}
5 \\ 6
\end{pmatrix}\times
&&\begin{pmatrix}
5 \\ 6
\end{pmatrix}\times
&&\;\dots
\\
& && && && && \;\bigcup && && && && && \\
&\quad\vdots\qquad &&\quad\vdots\qquad && \quad\vdots\qquad &&\quad\vdots\qquad &&\quad\vdots\qquad && \quad\vdots\qquad &&\quad\vdots\qquad &&\quad\vdots\qquad &&\quad\vdots\qquad && \;\;\vdots\qquad
\end{alignat*}
We can see that the DS dimension of this class is infinity. This is because for any $d > 4$, $\mcH$ restricted to $(4,5,\dots,d)$ definitely contains $\{1,2\}^{d-3}$. Thus, $\mcH$ is not PAC learnable. However, observe that the $2$-DS dimension of this class is just 3. We can see that the sequence $(1,2,3)$ is $2$-DS shattered --- but this is the largest sequence that is $2$-DS shattered by $\mcH$. This is because, for any sequence that contains $x \in \{4,5,\dots\}$, for a fixed labelling of the points in the sequence other than $x$, we can only obtain two different labels on $x$ when restricting $\mcH$ to such a sequence (and so, we cannot obtain more than one neighbor in the direction of $x$). Hence, such a sequence cannot be $2$-DS shattered. Note that \Cref{thm:sufficiency} still implies that $\mcH$ is \textit{2-list PAC learnable}!

This example can easily be generalized to show the existence of hypothesis classes having finite $k$-DS dimension but infinite $k'$-DS dimension for any $k > k'$. 
\end{example}

We now also define other related dimensions from the literature, and their list analogues. Importantly, each of these turns out to be the \textit{right} analogue for the purposes of our subsequent analysis.

\begin{definition}[Natarajan dimension \cite{natarajan1989learning}]
    \label{def:natarajan-dim}
    A hypothesis class $\mcH \subseteq \mcY^{\mcX}$ Natarajan shatters a sequence $S \in \mcX^d$ if there exist $2$-lists $y_i \in \{Y \subseteq \mcY: |Y|=2\}$, $i=1,\dots,d$ such that $\mcH|_{S} \supseteq \prod_{i=1}^{d}y_i$, where $\prod_{i=1}^{d}y_i$ denotes the Cartesian product of the 2-lists $y_1,\dots,y_d$. The Natarajan dimension of $\mcH$, denoted as $d_{N}=d_{N}(\mcH)$, is the largest integer $d$ such that $\mcH$ Natarajan shatters some sequence $S \in \mcX^d$.
\end{definition}

As mentioned above, the Natarajan dimension was known to completely characterize learnability in the case when the number of labels $|\mcY|$ is finite. That it is not sufficient to characterize learnability when the number of labels $|\mcY|=\infty$ was only recently shown by \cite{brukhim2022characterization}.

The list analogue of the Natarajan dimension is defined ahead. This quantity was termed as ``$k$-dimension" in \cite{daniely2015inapproximability} -- the two definitions are equivalent.
\begin{definition}[$k$-Natarajan dimension]
    \label{def:k-natarajan-dim}
    A hypothesis class $\mcH \subseteq \mcY^{\mcX}$ $k$-Natarajan shatters a sequence $S \in \mcX^d$ if there exist $(k+1)$-lists $y_i \in \{Y \subseteq \mcY: |Y|=k+1\}$, $i=1,\dots,d$ such that $\mcH|_{S} \supseteq \prod_{i=1}^{d}y_i$. The $k$-Natarajan dimension of $\mcH$, denoted as $\dnat=\dnat(\mcH)$, is the largest integer $d$ such that $\mcH$ $k$-Natarajan shatters some sequence $S \in \mcX^d$.
\end{definition}

\begin{observation}
    \label{obs:dnat_le_dds}
    $d_N(\mcH) \le d_{DS}(\mcH)$ and $\dnat(\mcH) \le \dds(\mcH)$, for any $\mcH$.
\end{observation}

The last dimension that we care about is the exponential dimension.
\begin{definition}[Exponential dimension \cite{brukhim2022characterization}]
    \label{def:exp-dim}
    We say that a hypothesis class $\mcH \subseteq \mcY^{\mcX}$ exponential shatters a sequence $S \in \mcX^d$ if $|\mcH|_{S}| \ge 2^{d}$. The exponential dimension of a hypothesis class $\mcH \subseteq \mcY^{\mcX}$, denoted as $d_{E}=d_{E}(\mcH)$, is the largest integer $d$ such that $\mcH$ exponential shatters some sequence $S \in \mcX^d$.
\end{definition}

The natural list generalization is:
\begin{definition}[$k$-exponential dimension]
    \label{def:k-exp-dim}
    We say that a hypothesis class $\mcH \subseteq \mcY^{\mcX}$ $k$-exponential shatters a sequence $S \in \mcX^d$ if $|\mcH|_{S}| \ge (k+1)^{d}$. The $k$-exponential dimension of a hypothesis class $\mcH \subseteq \mcY^{\mcX}$, denoted as $\dexp=\dexp(\mcH)$, is the largest integer $d$ such that $\mcH$ $k$-exponential shatters some sequence $S \in \mcX^d$.
\end{definition}

Now, we define an important combinatorial object associated with a hypothesis class, namely the one-inclusion graph.
\subsection{One-inclusion graphs}
\label{sec:oig}

\begin{definition}[One-inclusion graph \cite{haussler1994predicting, rubinstein2006shifting}]
    \label{def:oig}
    The one-inclusion graph of $\mcH \subseteq \mcY^m$ is a hypergraph $\mcG(\mcH)=(V,E)$ that is defined as follows. The vertex-set is $V=\mcH$. For each $i \in [m]$ and $f:[m]\setminus\{i\} \to\mcY$, let $e_{i,f}$ be the set of all $h \in \mcH$ that agree with $f$ on $[m]\setminus \{i\}$. The edge-set is
    \begin{equation}
        \label{eqn:oig-edge-set}
        E = \{e_{i,f}: i \in [m], f:[m] \setminus \{i\} \to \mcY, e_{i,f}\neq \emptyset \}.
    \end{equation}
    We say that the edge $e_{i,f}\in E$ is in the direction $i$, and is adjacent to/contains the hypothesis/vertex $h$ if $h \in e_{i,f}$. Every vertex $h\in V$ is adjacent to exactly $m$ edges. The size of the edge $e_{i,f}$ is the size of the set $|e_{i,f}|$. 
\end{definition}

With respect to the one-inclusion graph, one can think about the \textit{degrees}\footnote{All these degrees are once again natural generalizations of the degrees defined in \cite{daniely2014optimal, brukhim2022characterization}.} of its vertices.

\begin{definition}[$k$-degree]
    \label{def:k-degree}
    Let $\mcG(\mcH)=(V,E)$ be the one-inclusion graph of $\mcH \subseteq \mcY^m$. The $k$-degree of a vertex $v \in V$ is
    \begin{align*}
        \degree(v) &= |\{e \in E: v \in e, |e|>k\}|.
    \end{align*}
\end{definition}

\begin{definition}[Average $k$-degree]
    \label{def:avg-k-degree}
    Let $\mcG(\mcH)=(V,E)$ be the one-inclusion graph of $\mcH \in \mcY^m$. The average $k$-degree of $\mcH$ is
    \begin{align*}
        \avd(\mcH) = \frac{1}{|V|}\sum_{v \in V}\degree(v) = \frac{1}{|V|}\sum_{e \in E:|e|>k}|e|.
    \end{align*}
\end{definition}

We now define the notion of \textit{orienting} edges of a one-inclusion graph to lists of vertices they are adjacent to. As alluded to earlier, an orientation corresponds to the behavior of a (deterministic) learning algorithm while making predictions on an unlabeled test point, given a set of labeled points as input.
\begin{definition}[List orientation]
    \label{def:list-orientation}
    A list orientation $\sigmak$ of the one-inclusion graph $\mcG(\mcH) = (V,E)$ having list size $k$ is a mapping $\sigmak:E \to \{V'\subseteq V: |V'|\le k\}$ such that for each edge $e \in E$, $\sigmak(e) \subseteq e$.
\end{definition}

The $k$-outdegree of a list orientation $\sigmak$ is defined as:
\begin{definition}[$k$-outdegree of a list orientation]
    Let $\mcG(\mcH)=(V,E)$ be the one-inclusion graph of a hypothesis class $\mcH$, and let $\sigmak$ be a $k$-list orientation of it. The $k$-outdegree of $v \in V$ in $\sigmak$ is
    \begin{equation}
        \label{eqn:vertex-k-out-degree}
        \outdeg(v;\sigmak) = |\{e:v \in e, v \notin \sigmak(e)\}|.
    \end{equation}
    The maximum $k$-outdegree of $\sigmak$ is
    \begin{equation}
        \label{eqn:max-k-outdegree}
        \outdeg(\sigmak) = \sup_{v \in V} \outdeg(v;\sigmak).
    \end{equation}
\end{definition}

The following lemma is a nice introductory demonstration of how a bound on the $k$-DS dimension of a hypothesis class helps one greedily construct a list orientation of small maximum $k$-outdegree for its associated one-inclusion graph. The motif behind this greedy construction will be recurring ahead.

\begin{lemma}
    \label{lem:dds+1-outdeg-dds}
    If $\mcH \subseteq \mcY^{d+1}$ has $k$-DS dimension at most $d$, then there exists a $k$-list orientation $\sigmak$ of $\mcG(\mcH)$ with $\outdeg(\sigmak)\le d$.
\end{lemma}
\begin{proof}
    We prove this by induction on the size of $\mcH$. For the base case, we have $|\mcH|=1$. In this case, say $\mcH = \{h\}$, so that $\mcG(\mcH)$ will simply have $d+1$ singleton edges adjacent to $h$. We can have $\sigmak$ orient each of these edges towards the singleton list $\{h\}$, and this ensures that $\outdeg(\sigmak)=\outdeg(h;\sigmak)=0$. For the inductive step, let $|\mcH|>1$, and assume that the claim holds for all $\mcH'$ having $|\mcH'| \le |\mcH|-1$. Since $\dds(\mcH)\le d$, it must be the case that there exists $h \in \mcH$ which satisfies $|\{ e \in E: h \in e, |e|>k\}|\le d$. If not, every $h \in \mcH$ has at least $k$ $i$-neighbors in every direction $i \in [d+1]$, which means that the $k$-DS dimension of $\mcH$ would be $d+1$, contradicting the supposition that $\dds(\mcH)\le d$. Now, consider the hypothesis class $\mcH' = \mcH \setminus \{h\}$. Edges in $\mcG(\mcH')$ are obtained by deleting $h$ from the edges in $\mcG(\mcH)$ i.e., every edge in $\mcG(\mcH)$ has a counterpart in $\mcG(\mcH')$ (the only edges that do not have a counterpart are the singleton edges adjacent to $h$). The inductive hypothesis ensures that there exists a $k$-list orientation $\tilde{\sigma}^k$ of $\mcG(\mcH')$ with $\outdeg(\tilde{\sigma}^k)$ at most $d$. We will construct the required $k$-list orientation $\sigmak$ of $\mcG(\mcH'\cup\{h\})$ from $\tilde{\sigma}^k$ as follows: let us first think of all the edges $e$ that have a counterpart $e'$. For the edges that do not agree with $h$, $e=e'$, and we let $\sigmak(e)=\tilde{\sigma}^k(e')$. These edges do not contribute to the $k$-outdegree of $h$. For the edges $e$ that agree with $h$, we have that $e=e'\cup\{h\}$. But we know that at most $d$ of these edges $e$ have size larger than $k$. For these edges, we let $\sigmak(e)=\tilde{\sigma}^k(e')$ --- these will be the only edges contributing to the $k$-outdegree of $h$ in $\sigmak$. For the other edges $e$, we know that their size is at most $k$, and hence the size of their counterparts $e'$ was at most $k-1$. Thus, in $\tilde{\sigma}^k$, these edges $e'$ could only have been oriented to lists of size at most  $k-1$. For these edges $e$, we set $\sigmak(e) = \tilde{\sigma}^k(e') \cup \{h\}$. Finally, we orient all the singleton edges that $h$ is a part of in $\mcG(\mcH)$ to $\{h\}$. We can see that the $k$-outdegree of every vertex other than $h$ in $\sigmak$ is the same as it was in $\tilde{\sigma}^k$. The $k$-outdegree of $h$ is by construction at most $d$ (one for each edge that has size larger than $k$), which completes the inductive argument. The extension to the case when $|\mcH|=\infty$ requires a compactness argument, and is given in \Cref{sec:orientations-for-infinite-graphs-appendix}. 
\end{proof}

\section{Finite \texorpdfstring{$k$}{k}-DS dimension is necessary}
\label{sec:necessity}

In this section, we will prove \Cref{thm:necessity}, thereby establishing that finiteness of the $k$-DS dimension is a necessary property for a hypothesis class $\mcH$ to be $k$-list PAC learnable. We first prove necessity in the transductive list learning model. In this model, the \textit{transductive error rate} of any list learning algorithm $\mcA$ is the function $\epsilon^{k}_{\mcA, \mcH}:\N \to [0,1]$ defined as
\begin{equation}
    \label{eqn:trans-algo-error}
    \epsilon^k_{\mcA, \mcH}(m) = \sup_{\mcD}\sup_{h^* \in \mcH} \Pr_{(S,x) \sim \mcD^m}[\mcA(S, h^*|_S, x) \not\owns h^*(x)],
\end{equation}
where $\mcD$ is a distribution over $\mcX$ and $\mcA(S, h^*|_S, x)$ is the list output by $\mcA$ on the unlabeled point $x$, given input sample $S$ (of size $m-1$) labeled by $h^*$. The transductive error rate of the class $\mcH$ is then defined as
\begin{equation}
    \label{eqn:trans-error}
    \transerror(m) = \inf_{\mcA} \epsilon^k_{\mcA, \mcH}(m).
\end{equation}

\cite[Theorem 2]{daniely2014optimal} show that finite DS dimension is necessary for multiclass PAC learning; in what follows, we will generalize their lower bound argument to the list learning setting. Towards this, let us first define the notion of \textit{maximal average $k$-degree} of the one-inclusion graph $\mcG(\mcH)$.
\begin{definition}[Maximal average $k$-degree]
    \label{def:maximal-avg-k-degree}
    The maximal average $k$-degree of a hypothesis class $\mcH \subseteq \mcY^{\mcX}$ is
    \begin{align*}
        \md(\mcH) = \max_{\mcF \subseteq \mcH, |\mcF|<\infty} \avd(\mcF).
    \end{align*}
\end{definition}

We can now define the \textit{density} of the hypothesis class $\mcH$ as:
\begin{equation}
    \label{eqn:muh-definition}
    \muh(m) = \max_{S \in \mcX^m}[\md(\mcH|_S)].
\end{equation}

\begin{observation}
    \label{obs:dds-max-muh}
    $\dds(\mcH) = \max\{m:\muh(m)=m\}$.
\end{observation}
\begin{proof}
    If $\muh(m)=m$, there exists $S \in \mcX^m$ and $\mcF \subseteq \mcH, |\mcF| < \infty$ such that
    \begin{align*}
        &\avd(\mcF|_S) = m \\
        \implies \qquad & \degree(f) = m,\; \forall f \in \mcF|_S,
    \end{align*}
    which means that every $f \in \mcF|_S$ has an edge of size $> k$ in each of the $m$ directions, which by definition means that $S$ is $k$-DS shattered by $\mcF$, and hence by $\mcH$.
    On the other hand, if $\dds(\mcH)=m$, there exists a finite $\mcF \subseteq \mcH$ which $k$-DS shatters some sequence $S \in \mcX^m$, which means that $\avd(\mcF|_S)=m$, and hence $\muh(m)=m$.
\end{proof}

The following theorem states that the density lower bounds the transductive error rate of $\mcH$.
\begin{theorem}[Transductive model lower bound]
    \label{thm:trans-error-lb}
    For every hypothesis class $\mcH$,
    \begin{align*}
        \frac{\muh(m)}{e(k+1)m} \le \transerror(m).
    \end{align*}
\end{theorem}
\begin{proof}
    Let $S^* = \{x_1,\dots,x_m\} \in \mcX^m$ be a sequence that realizes $\muh(m)=\md(\mcH|_{S^*})$. Further, let $\mcF \subseteq \mcH, |\mcF|<\infty$ be such that $\md(\mcH|_{S^*}) = \avd(\mcF|_{S^*})$. Let $\mcD=\mathrm{Unif}(S^*)$ be the uniform distribution on the elements in $S^*$. We will obtain the required bound via the probabilistic method. Let us draw $h^*$ uniformly at random from $\mcF|_{S^*}$. We will show that any list learning algorithm $\mcA$ satisfies
    \begin{equation}
        \label{eqn:prob-method-lb}
        \E_{h^* \sim \mathrm{Unif}(\mcF|_{S^*})} \E_{(S,x) \sim \mcD^m}\left[\mathds{1}[\mcA(S, h^*|_S, x) \not\owns h^*(x)] \right] \ge \frac{1}{e(k+1)m}\muh(m).
    \end{equation}
    Fix an edge $e \in \mcG(\mcF|_{S^*})$ that satisfies $|e|>k$. Say this edge is in the $i^{\text{th}}$ direction for some $i \in [m]$. We will think of how we can map this edge to instantiations of the learning problem for which the any algorithm $\mcA$ has large error. Assume the following event $\mathrm{E_e}$ happens: (1) the unlabeled test point $x$ happens to be $x_i$ (where $i$ is the direction of edge $e$), (2) the sample $S=\{x'_1,\dots,x'_{m-1}\}$ happens to be such that $x'_j \neq x_i$ for any $j\in[m-1]$, and (3) $h^*$ happens to be one of the $\ge k+1$ functions belonging to $e$. Conditioned on this event, the input that the algorithm receives is always identical, but the ground-truth label on the test point $x$ can be any one of $\ge k+1$ values. Thus, conditioned on this event, any algorithm $\mcA$ (deterministic/randomized) that outputs a list of  has expected error at least $\frac{1}{k+1}$. Thus, we can iterate over all such events $E_e$, and bound the LHS of \cref{eqn:prob-method-lb} as
    \begin{align*}
        \E_{h^* \sim \mathrm{Unif}(\mcF|_{S^*})} \E_{(S,x) \sim \mcD^m}\left[\mathds{1}[\mcA(S, h^*|_S, x) \not\owns h^*(x)] \right] &\ge \sum_{e:|e|>k}\Pr[E_e]\cdot\frac{1}{k+1}.
    \end{align*}
    Now, the probability of (1) above is $\frac{1}{m}$, that of (2) is $\left(1-\frac{1}{m}\right)^{m-1}$ and that of (3) is $\frac{|e|}{|\mcF|_{S^*}|}$. As event $E_e$, which is the (independent) conjunction of (1), (2) and (3), we get
    \begin{align*}
        \E_{h^* \sim \mathrm{Unif}(\mcF|_{S^*})} \E_{(S,x) \sim \mcD^m}\left[\mathds{1}[\mcA(S, h^*|_S, x) \not\owns h^*(x)] \right] &\ge \sum_{e:|e|>k}\frac{1}{m}\cdot \left(1-\frac{1}{m}\right)^{m-1} \cdot \frac{|e|}{|\mcF|_{S^*}|}\cdot\frac{1}{k+1} \\
        &= \frac{1}{m}\cdot \left(1-\frac{1}{m}\right)^{m-1}\cdot\frac{1}{k+1}\cdot \frac{1}{|\mcF|_{S^*}|}\sum_{e:|e|>k}|e| \\
        &= \frac{1}{(k+1)m}\cdot \left(1-\frac{1}{m}\right)^{m-1} \avd(\mcF|_{S^*}) \\
        &\ge \frac{\avd(\mcF|_{S^*})}{e(k+1)m} \\
        &= \frac{\muh(m)}{e(k+1)m}.
    \end{align*}
    Thus, the probabilistic method implies that there exists an $h^* \in \mcF|_{S^*}$ (and hence a corresponding $h^* \in \mcH$) satisfying the above inequality, which completes the proof.
\end{proof}


Now, if $\dds(\mcH) = \infty$, \Cref{obs:dds-max-muh} implies that $\muh(m)=m,\; \forall m \in \N$. \Cref{thm:trans-error-lb} then implies that for any $m$, $\transerror(m) \ge \frac{1}{e(k+1)}$. Consequently, for every fixed value of $k$, if $\dds(\mcH)=\infty$,  the (transductive) sample complexity of learning $\mcH$ is $\infty$.

Finally, the lower bound on the transductive error rate in \Cref{thm:trans-error-lb} also helps us lower bound the PAC sample complexity, thereby completing the proof of \Cref{thm:necessity}.
\begin{corollary}[PAC model lower bound]
    \label{cor:pac-sample-cxty-lb}
    If $\dds(\mcH)=\infty$, the PAC sample complexity for $k$-list learning $\mcH$ is $\infty$.
\end{corollary}
\begin{proof}
    Fix any $k$-list learning algorithm $\mcA$. Let $m_{\mcA, \mcH}\left(\frac{\epsilon}{2}, \frac{\epsilon}{2}\right)$ be the PAC sample complexity of $\mcA$ to $\left(\frac{\epsilon}{2}, \frac{\epsilon}{2}\right)$-learn $\mcH$. For the $\mcD$, $h^*$ given in the proof of \Cref{thm:trans-error-lb}, let us draw $m_{\mcA, \mcH}\left(\frac{\epsilon}{2}, \frac{\epsilon}{2}\right)$ samples from $\mcD$ labeled by $h^*$, and an additional unlabeled test point $x$ as input to $\mcA$. By the PAC learning guarantee, the probability that $\mcA$ errs on $x$ is at most $\left(1-\frac{\epsilon}{2}\right)\cdot\frac{\epsilon}{2}+\frac{\epsilon}{2}\cdot 1 = \epsilon$. But the lower bound in \Cref{thm:trans-error-lb} requires that
    \begin{align*}
        m_{\mcA, \mcH}\left(\frac{\epsilon}{2}, \frac{\epsilon}{2}\right) \ge \min\left\{ m \big| \frac{\muh(m+1)}{e(k+1)(m+1)} \le \epsilon \right\}.
    \end{align*}
    If $\dds(\mcH) = \infty$, \Cref{obs:dds-max-muh} implies that $\muh(m)=m$ for all $m \in \N$. Hence, it must be the case that
    \begin{align*}
         m_{\mcA, \mcH}\left(\frac{\epsilon}{2}, \frac{\epsilon}{2}\right) \ge \min\left\{ m \big| \frac{1}{e(k+1)} \le \epsilon \right\}.
    \end{align*}
    For the choice of $\epsilon=\frac{1}{e(k+2)}$, the RHS above is $\infty$, which completes the proof.
\end{proof}
\section{A generalization of Sauer's lemma for lists}
\label{sec:list-sauer-lemma}
The VC dimension of a binary hypothesis class $\mcH$ is the largest sequence size $d$ such that $\mcH$ restricted to a sequence of size $d$ has size $2^d$. When we start restricting $\mcH$ to sequences of size $m$ larger than $d$, how does the size of the restricted class grow as a function of $m$? Sauer's lemma \cite{sauer1972density, shalev2014understanding} states that when $m$ is larger than $d$, the size of $\mcH$ restricted to sequences of size $m$ is only a polynomial function in $m$ (as opposed to possibly being exponential for $m \le d$). This polynomial behavior crucially enables establishing uniform convergence (and PAC learnability) properties for binary classes with finite VC dimension. In the case of hypothesis classes with more than $2$ classes, \cite{haussler1995generalization} generalize Sauer's lemma, and show an analogous result in terms of the Natarajan dimension of the class. Naturally, in the list learning setting, we are interested in how the size of hypothesis classes having $k$-Natarajan $\dnat$ grows, when restricted to sequences of size larger than $\dnat$.

By carefully extending the techniques in \cite{haussler1995generalization}, we prove a generalization of Sauer's lemma for lists. A qualitatively identical theorem was proven in an altogether different problem setting of deriving inapproximability results in combinatorial auctions in \cite[Theorem 1.5]{daniely2015inapproximability}. In fact, it is worth pointing out that in both our generalization and \cite[Theorem 1.5]{daniely2015inapproximability}, the size of the hypothesis class no longer grows only polynomially with $m$, but has an exponential dependence. Therefore, such a generalization might a priori not appear useful in proving learning results. Even so, we successfully leverage the list-variant of Sauer's lemma in the context of list PAC learning, and include our (slightly structurally different) proof below for completeness.
\begin{theorem}[Sauer's lemma for lists]
\label{thm:list-sauer-lemma}
Let $d,m \ge 0$, $d \le m$, $k \ge 2$ and let $N_1,\dots,N_m > k+1$. Let
\begin{align*}
    F \subseteq \prod_{i=1}^m \{1,\dots,N_i\}
\end{align*}
be such that $d^k_{N}=d^k_{N}(F)\le d$. Let us denote the set of all $i$-sized subsets of $[m]$ as $\Gamma_{m,i} := \{S \subseteq [m]: |S|=i\}$. Then, we have that
\begin{align*}
    |F| \le k^{m-d}\sum_{i=0}^d\sum_{S \in \Gamma_{m,i}}\prod_{j \in S}\binom{N_j}{k+1}.
\end{align*}
\end{theorem}
\begin{proof}
We will prove this via a double induction on $m$ and $d$.

\paragraph{Base Case 1:} $d=0$. \\
The bound reduces to showing
\begin{align*}
    |F| \le k^m.
\end{align*}
Assume for the sake of contradiction that $|F| > k^m$. Then, it must be the case that there is a coordinate $i \in [m]$ at which members of $F$ attain at least $k+1$ different values. But then this coordinate witnesses $d^k_N(F)\ge 1 > d$, which is a contradiction.

\paragraph{Base Case 2:} $d=m$. \\
The bound reduces to showing
\begin{align*}
    |F| \le \sum_{i=0}^m\sum_{S \in \Gamma_{m,i}}\prod_{j \in S}\binom{N_j}{k+1}.
\end{align*}
Let us compute the cardinality of the set $X=\prod_{i=1}^m \{1,\dots,N_i\}$. Let us partition the elements of $X$ into disjoint sets $X_i$ based on the number $i$ of non-1 elements i.e. $i$ can range from $0$ to $m$. The number of elements having exactly $i$ non-1's is $\sum_{S \in \Gamma_{m,i}}\prod_{j \in S}(N_j -1)$. Thus, we have
\begin{align*}
    |X| &= \sum_{i=0}^m|X_i| \\
    &= \sum_{i=0}^m\sum_{S \in \Gamma_{m,i}}\prod_{j \in S}(N_j -1) \\
    &\le \sum_{i=0}^m\sum_{S \in \Gamma_{m,i}}\prod_{j \in S}\binom{N_j}{k+1} \qquad (\text{since each $N_j > k+1$})
\end{align*}
Since $F \subseteq X$, we get the required bound on $|F|$.

\paragraph{Inductive step:} Now, given $0 < d < m$, and $F \subseteq \prod_{i=1}^m\{1,\dots,N_i\}$ such that $d^k_N(F) \le d$, let us assume the following induction hypothesis: for any $0\le d'\le m-1$ and $F' \subseteq \prod_{i=1}^{m-1}\{1,\dots,N_i\}$ such that $d^k_{N}(F')\le d'$,
\begin{equation}
    \label{eqn:sauer-inductive-hypothesis}
    |F'| \le k^{m-1-d'}\sum_{i=0}^{d'}\sum_{\Gamma_{m-1,i}}\prod_{j \in S}\binom{N_j}{k+1}.
\end{equation}
We will show that this implies $|F| \le k^{m-d}\sum_{i=0}^d\sum_{S \in \Gamma_{m,i}}\prod_{j \in S}\binom{N_j}{k+1}$.

\noindent For $f \in F$, denote its coordinates as $f_1,\dots,f_m$. Define:
\begin{align*}
    &\pi(f)=(f_1,\dots,f_{m-1}) \qquad \text{($\pi$ simply truncates the last coordinate of $f$)} \\
    &\\
    &\alpha_{1}(f_1,\dots,f_{m-1}) = \min\{v: (f_1,\dots,f_{m-1},v)\in F\} \\
    &\alpha_2(f_1,\dots,f_{m-1}) = \min\{v: (f_1,\dots,f_{m-1},v)\in F, v > \alpha_1(f_1,\dots,f_{m-1})\}  \\
    &\vdots \\
    &\alpha_{k}(f_1,\dots,f_{m-1}) = \min\{v:(f_1,\dots,f_{m-1},v)\in F, v > \alpha_{k-1}(f_1,\dots,f_{m-1})\} \\
    &\\
    &F_{1}=\{(f_1,\dots,f_m) \in F: f_m = \alpha_1(f_1,\dots,f_{m-1})\} \\
    &F_{2}=\{(f_1,\dots,f_m) \in F: f_m = \alpha_2(f_1,\dots,f_{m-1})\} \\
    &\vdots\\
    &F_{k}=\{(f_1,\dots,f_m) \in F: f_m = \alpha_{k}(f_1,\dots,f_{m-1})\}
\end{align*}
Finally, for $1 \le v_1 < v_2 < \dots < v_{k+1} \le N_m$, define
\begin{align*}
    &F_{v_1,\dots,v_{k+1}} = \\ &\left\{(f_1,\dots,f_m) \in F \setminus \bigcup_{i=1}^{k}F_i: f_m = v_{k+1}, \alpha_1(f_1,\dots,f_{m-1})=v_1,\dots,\alpha_{k}(f_1,\dots,f_{m-1})=v_{k}\right\}
\end{align*}
Then, we can see that all the sets $F_i, F_{v_1,\dots,v_{k+1}}$ are disjoint, and further that
\begin{align}
    &F = \left(\bigcup_{i=1}^{k}F_i\right) \cup \left(\bigcup_{1 \le v_1 < v_2 < \dots < v_{k+1} \le N_m}F_{v_1,\dots,v_{k+1}}\right) \nonumber \\
    \implies\qquad&|F|= \sum_{i=1}^{k}|F_i| + \sum_{1\le v_1<v_2<\dots<v_{k+1}\le N_m} \left|F_{v_1,\dots,v_{k+1}}\right|. \label{eqn:F-decomposition}
\end{align}
\paragraph{Bounding $|F_i|$ terms:} First, we will show that each $|F_i|$ in \cref{eqn:F-decomposition} is bounded under the inductive hypothesis \cref{eqn:sauer-inductive-hypothesis}. Fix $1 \le i \le k$.
Consider $\pi(F_{i})$. For $f \neq g \in F_i$, it must be the case that $\pi(f)\neq \pi(g)$. This is because, if $\pi(f)=\pi(g)$, then $f_m \neq g_m$, which is not possible. 
Thus, for every $f \neq g \in F_i$, we have $\pi(f) \neq \pi(g)$, implying that there is a one-to-one mapping between $\pi(F_i)$ and $F_i$, or equivalently, $|\pi(F_i)|=|F_i|$.
But now, $\pi(F_i)$ is a set of functions on $m-1$ coordinates. Its $k$-Natarajan dimension is at most that of $F_i$, which is at most $d$. Instantiating the inductive hypothesis \cref{eqn:sauer-inductive-hypothesis} with $d'=d\le m-1$ and $F'=\pi(F_i)$, we get that
\begin{align}
    \label{eqn:Fi-bound}
    |F_i|=|\pi(F_i)| \le k^{m-1-d}\sum_{i=0}^d\sum_{S \in \Gamma_{m-1,i}}\prod_{j \in S}\binom{N_j}{k+1}.
\end{align}

\paragraph{Bounding $|F_{v_1,\dots,v_{k+1}}|$ terms:} Next, we will show that each $\left|F_{v_1,\dots,v_{k+1}}\right|$ in \cref{eqn:F-decomposition} is bounded. Fix $1\le v_1 < v_2 < \dots < v_{k+1} \le N_m$. We will show that the $k$-Natarajan dimension of $F_{v_1,\dots,v_{k+1}}$ is at most $d-1$.

Let $I \subseteq [m]$ be a set of indices that is $k$-Natarajan shattered by $F_{v_1,\dots,v_{k+1}}$ such that $|I|=l$. Then, $m \notin I$, since $f_m = v_{k+1}, ~\forall f \in F_{v_1,\dots,v_{k+1}}$. We will show that $I \cup \{m\}$ is $k$-Natarajan shattered by $F$. 

Let $y_1,\dots, y_l$ be $(k+1)$-lists that are witnesses to $F_{v_1,\dots,v_{k+1}}$ $k$-Natarajan shattering $I$. We will show that $y_1,\dots,y_l,y_{l+1}$ are witnesses to $F$ $k$-Natarajan shattering $I \cup \{m\}$ for $y_{l+1}=(v_1,\dots,v_{k+1})$.
For any $\bar{y}\in \prod_{i=1}^{l+1}y_i$, find $f=(f_1,\dots,f_{m-1},v_{k+1}) \in F_{v_1,\dots,v_{k+1}}$ that agrees with the first $l$ coordinates of $\bar{y}$ on $I$. For any $1\le j\le k$, we have that
$(f_1,\dots,f_{m-1},\alpha_j(f_1,\dots,f_{m-1})) \in F_{j}$ realizes the value $v_j$ on the coordinate $m$. Finally, $f$ itself realizes the value $v_{k+1}$ on the coordinate $m$. Thus, we have found members of $F$ realizing every configuration in $\prod_{i=1}^{l+1}y_i$ on $I \cup \{m\}$, and hence $F$ $k$-Natarajan shatters $I \cup \{m\}$. Since the $k$-Natarajan dimension of $F$ was at most $d$, and $m \notin I$, it must be the case that $|I| \le d-1$. Since $I$ was chosen arbitrarily, every set that is $k$-Natarajan shattered by $F_{v_1,\dots,v_{k+1}}$ must have size at most $d-1$, and hence $d^k_{N}(F_{v_1,\dots,v_{k+1}}) \le d-1$.

Now, consider $\pi(F_{v_1,\dots,v_{k+1}})$. If $f \neq g \in F_{v_1,\dots,v_{k+1}}$, it must be the case that they differ in the first $m-1$ coordinates, and hence $\pi(f)\neq\pi(g)$. Thus, we have a one-to-one mapping again between $\pi(F_{v_1,\dots,v_{k+1}})$ and $F_{v_1,\dots,v_{k+1}}$. But now, $\pi(F_{v_1,\dots,v_{k+1}})$ is a set of functions on $m-1$ coordinates. Its $k$-Natarajan dimension is at most that of $F_{v_1,\dots,v_{k+1}}$, which is at most $d-1$. Instantiating the inductive hypothesis \cref{eqn:sauer-inductive-hypothesis} with $d'=d-1 < m-1$ and $F'=\pi(F_{v_1,\dots,v_{k+1}})$, we get that
\begin{align}
    |F_{v_1,\dots,v_{k+1}}| = |\pi(F_{v_1,\dots,v_{k+1}})| &\le k^{m-1-(d-1)}\sum_{i=0}^{d-1}\sum_{S \in \Gamma_{m-1,i}}\prod_{j \in S}\binom{N_j}{k+1} \nonumber\\
    &= k^{m-d}\sum_{i=0}^{d-1}\sum_{S \in \Gamma_{m-1,i}}\prod_{j \in S}\binom{N_j}{k+1}. \label{eqn:F-v1-vk-bound}
\end{align}

\paragraph{Putting together:} Substituting the bounds in \cref{eqn:Fi-bound} and \cref{eqn:F-v1-vk-bound} in \cref{eqn:F-decomposition}, we get
\begingroup
\allowdisplaybreaks
\begin{align*}
    |F| &\le k\cdot k^{m-1-d}\sum_{i=0}^d\sum_{S \in \Gamma_{m-1,i}}\prod_{j \in S}\binom{N_j}{k+1} + \\
    &\qquad\qquad \sum_{1\le v_1<v_2<\dots<v_{k+1}\le N_m} k^{m-d}\sum_{i=0}^{d-1}\sum_{S \in \Gamma_{m-1,i}}\prod_{j \in S}\binom{N_j}{k+1} \\
    &=k^{m-d}\sum_{i=0}^d\sum_{S \in \Gamma_{m-1,i}}\prod_{j \in S}\binom{N_j}{k+1}+\binom{N_m}{k+1}k^{m-d}\sum_{i=0}^{d-1}\sum_{S \in \Gamma_{m-1,i}}\prod_{j \in S}\binom{N_j}{k+1} \\
    &= k^{m-d}\sum_{i=0}^d\sum_{S \in \Gamma_{m-1,i}}\prod_{j \in S}\binom{N_j}{k+1}+k^{m-d}\sum_{i=0}^{d-1}\sum_{S \in \Gamma_{m-1,i}}\prod_{j \in S\cup\{m\}}\binom{N_j}{k+1} \\
    &= k^{m-d}\sum_{i=0}^d\sum_{S \in \Gamma_{m-1,i}}\prod_{j \in S}\binom{N_j}{k+1}+k^{m-d}\sum_{i=1}^{d}\sum_{S \in \Gamma_{m-1,i-1}}\prod_{j \in S\cup\{m\}}\binom{N_j}{k+1} \\
    &= k^{m-d}\sum_{i=0}^d\sum_{S \in \Gamma_{m-1,i}}\prod_{j \in S}\binom{N_j}{k+1}+k^{m-d}\sum_{i=1}^{d}\sum_{S \in \Gamma_{m,i}, m \in S}\prod_{j \in S}\binom{N_j}{k+1} \\
    &= k^{m-d}\left[1+\sum_{i=1}^d\sum_{S \in \Gamma_{m-1,i}}\prod_{j \in S}\binom{N_j}{k+1}+\sum_{i=1}^{d}\sum_{S \in \Gamma_{m,i}, m \in S}\prod_{j \in S}\binom{N_j}{k+1}\right] \\
    &= k^{m-d}\left[1+\sum_{i=1}^d\sum_{S \in \Gamma_{m,i}, m \notin S}\prod_{j \in S}\binom{N_j}{k+1}+\sum_{i=1}^{d}\sum_{S \in \Gamma_{m,i}, m \in S}\prod_{j \in S}\binom{N_j}{k+1}\right] \\
    &= k^{m-d}\left[1+\sum_{i=1}^d\sum_{S \in \Gamma_{m,i}}\prod_{j \in S}\binom{N_j}{k+1}\right] \\
    &= k^{m-d}\sum_{i=0}^d\sum_{S \in \Gamma_{m,i}}\prod_{j \in S}\binom{N_j}{k+1},
\end{align*}
\endgroup
which completes the induction.
\end{proof}
\begin{corollary}
    \label{cor:list-sauer-lemma}
    For $k \ge 2$ and $p > k+1$, let $\mcH \subseteq [p]^m$ be such that $d^k_N(\mcH)=d^k_N\le m$. Then, we have that
    \begin{align*}
        |\mcH| \le k^{m-d^k_N}\sum_{i=0}^{d^k_N}\binom{m}{i}\binom{p}{k+1}^i.
    \end{align*}
\end{corollary}
\section{List orientations with bounded \texorpdfstring{$k$}{k}-outdegree}
\label{sec:shifting}

In this section, we will build up towards proving that the one-inclusion graph of a hypothesis class with bounded $k$-DS dimension (and hence bounded $k$-Natarajan dimension) has a $k$-list orientation with small maximum $k$-outdegree. 
The main theorem we will prove is:
\begin{theorem}
    \label{thm:outdeg-bound-in-terms-of-dnat}
    Let $\mcH \in [p]^m$ be a hypothesis class with $k$-Natarajan dimension $\dnat < \infty$. Then, there exists a $k$-list orientation $\sigmak$ of $\mcG(\mcH)$ with maximum $k$-outdegree
    \begin{align*}
        \outdeg(\sigmak) \le 240k^4\dnat \log(p).
    \end{align*}
\end{theorem}

\begin{remark}
    The proof methodology that follows also gives the same guarantee in the slightly more general (but essentially identical) setting where $\mcH \subseteq \mcY_1 \times \mcY_2 \times \dots \times \mcY_m$: $|\mcY_i|\le p \; \forall i \in [m]$.\footnote{Or one can think of mapping each $\mcY_i$ to $[p]$, getting the bound on the $k$-outdegree in the mapped space, and then remapping back to the $\mcY_i$ space (the one-inclusion graphs in the mapped/original space are isomorphic).}
\end{remark}
\cite{brukhim2022characterization} instantiate the combinatorial technique of shifting to construct orientations of small outdegree for hypothesis classes having finite Natarajan dimension. We prove \Cref{thm:outdeg-bound-in-terms-of-dnat} in a stepwise analogous manner. For a nice introduction to the shifting operator on a hypothesis class, we refer the reader to the beginning of Section 3 in \cite{brukhim2022characterization}. 

\begin{definition}[Shifting \cite{brukhim2022characterization}]
    \label{def:shifting}
    Let $\mcH \subseteq [p]^m$ and let $i \in [m]$. The shifting operator in the $i^{\text{th}}$ direction $\shift_i$ maps $\mcH$ to its shifted version $\shift_i(\mcH)$ as follows. Shifting is first defined on edges. For $f : [m] \setminus \{i\} \to [p]$, let $e_{f}$ be the collection of $h \in \mcH$ that agree with $f$ on $[m] \setminus \{i\}$. The shifting $\shift_i(e_f)$ is obtained by ``pushing
    $e_f$ downward''; namely, $\shift_i(e_f)$ is the collection of all $g \in [p]^m$ that agree with $f$ on $[m] \setminus \{i\}$ and $1 \leq g(i) \leq |e_f|$. The shifting of $\mcH$ is the union of all shifted edges
    $$\shift_i(\mcH) = \bigcup_{f} \shift_i(e_f) \subseteq [p]^m.$$
\end{definition}
The nice property about shifting is that the fixed point $\mcH_*$ of the shifting operations (i.e., $\mathbb{S}_i(\mcH_*) = \mcH_*$ for all $i \in [m]$) is guaranteed to exist, and is furthermore, closed downwards (i.e., if $h \in \mcH_* \subseteq [p]^m$, and $g_i \le h_i \, \forall i \in [m]$, then $g \in \mcH^*$).

We restate the following result from \cite{brukhim2022characterization} which guarantees that the shifting operation does not increase the size of a hypothesis class restricted to any sequence.
\begin{lemma}[Claim 22, \cite{brukhim2022characterization}]
\label{lem:shifting-does-not-increase-projections}
Let $\mcH \subseteq [p]^m$ and $i \in [m]$. For every $S \in [m]^k$,
\begin{align*}
    \left|\mathbb{S}_i(\mcH)|_S\right| \le \left|\mcH|_S\right|.
\end{align*}
\end{lemma}
This immediately leads to the following corollary.
\begin{corollary}
    \label{cor:shifting-decreases-exp-dim}
    For every $\mcH \subseteq [p]^m$ and $i \in [m]$,
    \begin{align*}
        \dexp(\mathbb{S}_i(\mcH)) \le \dexp(\mcH).
    \end{align*}
    Thus, $\dexp(\mcH_*)\le\dexp(\mcH)$. 
\end{corollary}

Next, we show that the shifting operation does not decrease a quantity which we denote as the \textit{shifting average $k$-degree}. This generalizes the analogous result (Claim 26) in \cite{brukhim2022characterization}.

\begin{definition}[Shifting average $k$-degree]
    \label{def:shifting-avg-k-degree}
    Let $\mcG(\mcH)=(V,E)$ be the one-inclusion graph of $\mcH \in [p]^m$. The shifting average $k$-degree of $\mcH$ is
    \begin{align*}
        \savd(\mcH) = \frac{1}{|V|}\sum_{e \in E}(|e|-k)_+, 
    \end{align*}
    where $(x)_+ = \max(x,0)$.
\end{definition}

\begin{lemma}[Shifting does not decrease $\savd$]
    \label{lem:savd-increases}
    For every $\mcH \in [p]^m$ and $i \in [m]$,
    \begin{align*}
        \savd(\shift_i(\mcH)) \ge \savd(\mcH).
    \end{align*}
    Thus, $\savd(\mcH_*) \ge \savd(\mcH).$
\end{lemma}
\begin{proof}
Since shifting does not change the number of vertices $|V|$, we only need to think about the summation over edges before/after shifting. The size of edges in the $i^{\text{th}}$ direction does not change during shifting, and hence we further only need to think about edges in other directions. Fix some $j \neq i$. We will enumerate over all edges in the $j^{\text{th}}$ direction by grouping all the members (vertices) in $\mcH$ according to their projection on $[m]\setminus \{i,j\}$. Concretely, fix $f \in \mcH|_{[m]\setminus\{i,j\}}$ and consider all vertices that agree with $f$ on $[m] \setminus \{i,j\}$. Encode this data by the $p \times p$ Boolean matrix $M_f$ defined by $M_f(a,b)=1$ iff adding $a,b$ to $f$ in the positions $i, j$ leads to a vertex in $\mcH$. The example below illustrates how $M_f$ changes during shifting for some fixed $f$, where we have $p=4$ and 10 hypotheses in $\mcH$ agree with $f$.
\begin{align*}
\begin{blockarray}{cccccc}
\begin{block}{c [cccc ] c}
   & 0 & 1 & 0 & 1 & \\
   & 1 & 0 & 1 & 1 &  \\
   & 1 & 0 & 1 & 1 &  \\
   & 1 & 0 & 0 & 1 &  \\
\end{block}
\mathrm{count} & \underline{3} & 1 & 2 & \underline{4} \\
\end{blockarray}
\Longrightarrow
\begin{blockarray}{cccccc}
\begin{block}{c [cccc ] c}
   & 1 & 1 & 1 & 1 & \\
   & 1 & 0 & 1 & 1 &  \\
   & 1 & 0 & 0 & 1 &  \\
   & 0 & 0 & 0 & 1 &  \\
\end{block}
\mathrm{count} & \underline{3} & 1 & 2 & \underline{4} \\
\end{blockarray}
\end{align*}
Let us collect the number of 1's in each column in the matrix $M_f$ in a count array. We can see that the count array corresponding to $M_f$ before and after shifting is identical. Let us denote the $k$ columns having the largest values in the count array as ``exposed'' columns, and the rest as ``unexposed''. In the example above, for $k=2$, columns 1 and 4 (underlined) are exposed, because they have the two largest count values (3 and 4) respectively. 

Every row $a$ in the matrix $M_f$ represents an edge (namely $f + \{i \to a\}$). The sum of $(|e|-k)_+$ over all edges $e$ that agree with $f$ can then be expressed as $\sum_{a=1}^p (M_f(a)-k)_+$, where $M_f(a)=\sum_{b=1}^pM_f(a,b)$. We can think of the quantity $(M_f(a)-k)_+$ as the contribution of the $a^{\text{th}}$ row. Before shifting (i.e., in the left matrix above), observe that for each row $a$, 
\begin{align*}
(M_f(a)-k)_+ \le M_f(a)-\sum_{b: b \text{ is exposed }}M_f(a,b).
\end{align*}
Thus, we can conclude that the sum of $(|e|-k)_+$ over all edges $e$ that agree with $f$ (before shifting) is at most
\begin{align*}
    \sum_{a=1}^pM_f(a)- \sum_{a=1}^p\sum_{b:  \text{is exposed}}M_f(a,b) = \sum_{a=1}^{p}\sum_{b: b \text{ is unexposed}}M_f(a,b).
\end{align*}
But the RHS above, which is the total number of unexposed elements in the matrix $M_f$ (before shifting), is precisely the sum over $(|e|-k)_+$ over all edges $e$ that agree with $f$ \textit{after} shifting. To see this, we only need to argue that every unexposed element gets counted in the contribution of its row after shifting. Observe that after shifting, if there exists an unexposed element in any row, it must be the case that all the entries in the exposed columns (exactly $k$ of them) in that row are 1 (by definition). Thus, these are the entries that account for the $k$ elements being subtracted in $(M_f(a)-k)_+$, which means that each unexposed element contributes in the positive to the sum $(M_f(a)-k)_+$.
\end{proof}

We now bound the average $k$-degree of $\mcG(\mcH)$ in terms of the $k$-exponential dimension of $\mcH$. This generalizes the analogous result (Proposition 27) in \cite{brukhim2022characterization}.
\begin{lemma}[$\avd$ is bounded by $\dexp$]
    \label{lem:avd-bounded-by-dexp}
    For every $\mcH \in [p]^m$, we have
    \begin{align*}
        \avd(\mcH) \le 4k^2\dexp(\mcH).
    \end{align*}
\end{lemma}
\begin{proof}
    We have that
    \begin{align}
        \avd(\mcH) &\le (k+1)\savd(\mcH) \nonumber \\
        &\le (k+1)\savd(\mcH_*) \qquad \text{(from \Cref{lem:savd-increases})} \nonumber \\
        &\le (k+1)\avd(\mcH_*) \label{eqn:avdH-avdH*}.
    \end{align}
    We will argue that $\avd(\mcH_*) \le (k+1)\dexp(\mcH_*)$. Let us induct on the size of $|\mcH_*|$. For the base case, we have that $|\mcH_*|=1$. In this case, $\avd(\mcH_*)=\dexp(\mcH_*)=0$, and so the inequality holds. Now, let $|\mcH_*|>1$, and let us assume that the inequality holds for all $\mcH'_*$ having size at most $|\mcH_*|-1$. Let $h_0$ be the concept in $\mcH_*$ such that no other concept in $\mcH_*$ is larger than $h_0$. Let us think of $h_0$ as a vector in $[p]^m$, and let $|h_0|$ be the number of coordinates in $h_0$ that are larger than $k$. Because $\mcH_*$ is closed downwards, all concepts $h \le h_0$ are contained in $\mcH_*$, which implies that $\degree(h_0)=|h_0|$. This also means $\mcH_*$ restricted to $|h_0|$ coordinates has size at least $(k+1)^{|h_0|}$. Consequently, $\dexp(\mcH_*) \ge |h_0|=\degree(h_0).$ Now, consider the hypothesis class $\mcH'_{*}=\mcH_* \setminus \{h_0\}$. This hypothesis class is also closed downwards, and the inductive hypothesis implies that $\avd(\mcH'_*)\le (k+1)\dexp(\mcH'_*)\le (k+1)\dexp(\mcH_*)$. Now, let us compare the sum of $k$-degrees of $h \in \mcH'_*$ in $\mcG(\mcH'_*)$ versus $\mcG(\mcH_*)$. In the worst case, $h_0$ was adjacent to $\degree(h_0)$ many edges each having size exactly $k+1$ in $\mcG(\mcH_*)$, in which case the sum of $k$-degrees of $h \in \mcH'_*$ is larger by $k\cdot\degree(h_0)$ in $\mcG(\mcH_*)$. Formally,
    \begin{align*}
        \sum_{v \in \mcG(\mcH_*)}\degree(v;\mcG(\mcH_*)) &= \sum_{v \in \mcG(\mcH_*), v \neq h_0}\degree(v;\mcG(\mcH_*)) + \degree(h_0;\mcG(\mcH_*)) \\
        &\hspace{-2cm}\le \sum_{v \in \mcG(\mcH'_*)}\degree(v; \mcG(\mcH'_*)) + k\cdot\degree(h_0;\mcG(\mcH_*)) + \degree(h_0;\mcG(\mcH_*)) \\
        &\hspace{-2cm}= \sum_{v \in \mcG(\mcH'_*)}\degree(v; \mcG(\mcH'_*)) + (k+1)|h_0| \\
        &\hspace{-2cm}\le (|\mcH'_*|)(k+1)\dexp(\mcH'_*) + (k+1)\dexp(\mcH_*) \qquad (\text{induction hypothesis and $|h_0|\le \dexp(\mcH_*)$}) \\
        &\hspace{-2cm}\le (|\mcH_*|-1)(k+1)\dexp(\mcH_*) + (k+1)\dexp(\mcH_*) \\
        &\hspace{-2cm}= |\mcH_*|(k+1)\dexp(\mcH_*),
    \end{align*}
    which gives us $\avd(\mcH_*) \le (k+1)\dexp(\mcH_*)$ as required. Recalling \cref{eqn:avdH-avdH*}, we get
    \begin{align*}
        \avd(\mcH) &\le (k+1)^2\dexp(\mcH_*) \\
        &\le (k+1)^2\dexp(\mcH) \qquad (\text{from \Cref{cor:shifting-decreases-exp-dim}}) \\
        &\le 4k^2\dexp(\mcH).
    \end{align*}
\end{proof}

The bound on $\avd$ helps us greedily construct a list orientation $\sigmak$ with bounded $k$-outdegree. The proof idea is similar to that of \Cref{lem:dds+1-outdeg-dds}.
\begin{corollary}
    \label{cor:outdeg-in-terms-of-dexp}
    For every $\mcH \in [p]^m$, there is a $k$-list orientation of $\mcG(\mcH)$ with maximum $k$-outdegree at most $4k^2\dexp(\mcH)$.
\end{corollary}
\begin{proof}
    We prove this by induction on the size of $\mcH$. For the base case, we have $|\mcH|=1$. In this case, say $\mcH = \{h\}$, so that $\mcG(\mcH)$ will simply have $m$ singleton edges adjacent to $h$. We can have $\sigmak$ orient each of these edges towards the singleton list $\{h\}$, and this ensures that $\outdeg(\sigmak)=\outdeg(h;\sigmak)=0$. For the inductive step, let $|\mcH|>1$, and assume that the claim holds for all $\mcH'$ having $|\mcH'| \le |\mcH|-1$. \Cref{lem:avd-bounded-by-dexp} guarantees that there exists a hypothesis $h$ in $\mcG(\mcH)$ with $k$-degree at most $4k^2\dexp(\mcH)$. Consider the hypothesis class $\mcH' = \mcH \setminus \{h\}$. Edges in $\mcG(\mcH')$ are obtained by deleting $h$ from the edges in $\mcG(\mcH)$ i.e., every edge in $\mcG(\mcH)$ has a counterpart in $\mcG(\mcH')$ (the only edges that do not have a counterpart are the singleton edges adjacent to $h$). The inductive hypothesis ensures that there exists a $k$-list orientation $\tilde{\sigma}^k$ of $\mcG(\mcH')$ with $\outdeg(\tilde{\sigma}^k)$ at most $4k^2\dexp(\mcH') \le 4k^2\dexp(\mcH)$. We will construct the required $k$-list orientation $\sigmak$ of $\mcG(\mcH'\cup\{h\})$ from $\tilde{\sigma}^k$ as follows: let us first think of all the edges $e$ that have a counterpart $e'$. For the edges that do not agree with $h$, $e=e'$, and we let $\sigmak(e)=\tilde{\sigma}^k(e')$. These edges do not contribute to the $k$-outdegree of $h$. For the edges $e$ that agree with $h$, we have that $e=e'\cup\{h\}$. But we know that at most $4k^2\dexp(\mcH)$ of these edges $e$ have size larger than $k$. For these edges, we let $\sigmak(e)=\tilde{\sigma}^k(e')$ --- these will be the only edges contributing to the $k$-outdegree of $h$ in $\sigmak$. For the other edges $e$, we know that their size is at most $k$, and hence the size of their counterparts $e'$ was at most $k-1$. Thus, in $\tilde{\sigma}^k$, these edges $e'$ could only have been oriented to lists of size at most $k-1$. For these edges $e$, we set $\sigmak(e) = \tilde{\sigma}^k(e') \cup \{h\}$. Finally, we orient all the singleton edges that $h$ is a part of in $\mcG(\mcH)$ to $\{h\}$. We can see that the $k$-outdegree of every vertex other than $h$ in $\sigmak$ is the same as it was in $\tilde{\sigma}^k$. The $k$-outdegree of $h$ is by construction at most $4k^2\dexp(\mcH)$ (one for each edge that has size larger than $k$), which completes the inductive construction.
\end{proof}

It remains to bound the $k$-exponential dimension in terms of the $k$-Natarajan dimension. This is where the list version of Sauer's lemma that we proved in \Cref{sec:list-sauer-lemma} above kicks in.
\begin{lemma}[Controlling the $k$-exponential dimension]
\label{lem:k-exp-dim-bound}
For every $\mcH \subseteq [p]^m$ having $d^k_N = d^k_N(\mcH)$ and $d^k_E = d^k_E(\mcH) < \infty$, we have
\begin{align*}
    d^k_E \le 60k^2d^k_N \log(p).
\end{align*}
\end{lemma}
\begin{proof}
    Consider the set $S$ which realizes $ d^k_E(\mcH)=d^k_E$, so that $|S|=d^{k}_E$. Then, we have that
    \begin{align*}
        |\mcH|_S| \ge (k+1)^{d^k_E}.
    \end{align*}
    Also, from \Cref{cor:list-sauer-lemma} above, we know that
    \begin{align*}
        |\mcH|_S| &\le k^{\dexp-\dnat}\sum_{i=0}^{d^k_N}\binom{\dexp}{i}\binom{p}{k+1}^i \\
        &\le k^{\dexp - \dnat}\left(\frac{pe}{k+1}\right)^{\dnat(k+1)} \left(\frac{\dexp e}{\dnat}\right)^{\dnat}.
    \end{align*}
    Thus, it must be the case that
    \begin{align}
        &(k+1)^{d^k_E} \le  k^{\dexp - \dnat}\left(\frac{pe}{k+1}\right)^{\dnat(k+1)} \left(\frac{\dexp e}{\dnat}\right)^{\dnat} \nonumber \\
        \implies\qquad &\left(\frac{k+1}{k}\right)^{\dexp} \le \left(\frac{1}{k}\right)^{\dnat}\left(\frac{pe}{k+1}\right)^{\dnat(k+1)} \left(\frac{\dexp e}{\dnat}\right)^{\dnat} = \left[\left(\frac{pe}{k+1}\right)^{k+1}\left(\frac{e}{k}\right)\left(\frac{\dexp}{\dnat}\right)\right]^{\dnat} \nonumber \\
        \implies\qquad &\frac{\dexp}{\dnat} \le \log_{\frac{k+1}{k}}\left[\left(\frac{pe}{k+1}\right)^{k+1}\left(\frac{e}{k}\right)\left(\frac{\dexp}{\dnat}\right)\right] = \frac{\log\left[\left(\frac{pe}{k+1}\right)^{k+1}\left(\frac{e}{k}\right)\left(\frac{\dexp}{\dnat}\right)\right]}{\log\left(1+\frac{1}{k}\right)} \nonumber \\
        \implies\qquad &\frac{\dexp}{\dnat} \le 2k\log\left[\left(\frac{pe}{k+1}\right)^{k+1}\left(\frac{e}{k}\right)\left(\frac{\dexp}{\dnat}\right)\right] \qquad \left(\text{since $\log\left(1+\frac{1}{k}\right)\ge \frac{1}{2k}$ for $k \ge 1$}\right) \nonumber \\ 
        \implies\qquad &\frac{\dexp}{\dnat}-2k\log\left(\frac{\dexp}{\dnat}\right) \le 2k(k+1)\log\left(\frac{pe}{k+1}\right) + 2k\log\left(\frac{e}{k}\right) \nonumber \\
        \implies\qquad &\frac{\dexp}{\dnat}-2k\log\left(\frac{\dexp}{\dnat}\right) \le 4k(k+1)\log\left(\frac{pe}{k+1}\right). \label{eqn:dexp-dnat-inequality} 
    \end{align}
    Now, assume for the sake of contradiction that $\frac{\dexp}{\dnat} > 10k(k+1)\log\left(\frac{pe}{k+1}\right) > 10k(k+1)$. For $k \ge 1$, in the regime $\frac{\dexp}{\dnat}\ge 10k(k+1)$, we have that 
    \begin{align*}
        \frac{\dexp}{\dnat}-2k\log\left(\frac{\dexp}{\dnat}\right) \ge \frac{1}{2}\cdot\frac{\dexp}{\dnat} > 5k(k+1)\log\left(\frac{pe}{k+1}\right),
    \end{align*}
    which is not possible, since we know that \cref{eqn:dexp-dnat-inequality} holds true. Here, we used that $\left(\text{denoting }\frac{\dexp}{\dnat} \text{ as }x\right)$ $2k\log(x) \le x/2$ when $x \ge 10k(k+1)$. Thus, it must be the case that
    \begin{align*}
        \frac{\dexp}{\dnat} &\le 10k(k+1)\log\left(\frac{pe}{k+1}\right) \\
        &\le 20k^2\log\left(pe\right) \le 60k^2\log(p).
    \end{align*}
\end{proof}

\begin{remark}
    \Cref{cor:outdeg-in-terms-of-dexp} and \Cref{lem:k-exp-dim-bound} together imply \Cref{thm:outdeg-bound-in-terms-of-dnat}.
\end{remark}

\section{Finite \texorpdfstring{$k$}{k}-DS dimension is sufficient}
\label{sec:sufficiency}

The one-inclusion algorithm \cite{haussler1994predicting, rubinstein2006shifting}, which is based on the one-inclusion graph (\Cref{sec:oig}), is an incredibly powerful algorithm, in that it essentially attains the \textit{optimal} sample complexity for \textit{any} learnable problem \cite{daniely2014optimal}. We describe a list version of the one-inclusion algorithm below, which forms the base for most of what follows:
\begin{algorithm}[H]
\caption{The one-inclusion list algorithm $\mcA_\mcH$ for $\mcH\subseteq\mcY^{\mcX}$} \label{algo:one-incl}
\hspace*{\algorithmicindent} 
\begin{flushleft}
  {\bf Input:} An $\mcH$-realizable sample $S = \big((x_1, y_1),\ldots,(x_m, y_m)\big)$. \\
{\bf Output:} A $k$-list hypothesis $\mcA_{\mcH}(S)=\mu^k_S:\mcX \to \{Y \subseteq \mcY: |Y|\le k\}$. \\
\ \\
For each $x \in \mcX$, the $k$-list $\mu^k_S(x)$ is computed as follows:
\end{flushleft}
\begin{algorithmic}[1]
\STATE Consider the class of all patterns over the \emph{unlabeled data}
{$\mcH|_{(x_1,\ldots,x_m,x)} \subseteq \mcY^{m+1}$}.
\STATE Find a $k$-list orientation $\sigmak$ of $\mcG(\mcH|_{(x_1,\ldots,x_m,x)})$ that \textit{minimizes} the \textit{maximum} $k$-outdegree.
\STATE Consider the edge in direction $m+1$ defined by $S$:
\[e =\{ h \in \mcH|_{(x_1,\ldots,x_m,x)} :  \forall i \in [m]  \ \ 
h(i) =y_i\}.\]
\STATE Set $\mu^k_S(x) = \{h(m+1):h \in \sigmak(e)\}$.
\end{algorithmic}
\end{algorithm}

The following guarantee holds for \Cref{algo:one-incl} as a straightforward implication of \Cref{lem:dds+1-outdeg-dds}.
\begin{claim}
    \label{claim:one-incl-gets-one-correct}
    Let $\mcH \subseteq \mcY^\mcX$ be so that $\dds=\dds(\mcH)  < \infty$. Let $\mcA_\mcH$ be \Cref{algo:one-incl}. For every $\mcH$-realizable sample $S=((x_1,y_1),\ldots,(x_{\dds+1},y_{\dds+1}))$, there exists $i \in [\dds+1]$ such that $\mu_{S_{-i}}(x_i) =y_i$, where $\mu_{S_{-i}} = \mcA_{\mcH}(S \setminus \{(x_i,y_i)\})$.
\end{claim}
\begin{proof}
    Let $\mcH'  = \mcH|_{(x_1,\ldots, x_{\dds+1})}$. We have that $\dds(\mcH') \le \dds(\mcH)=\dds$. Because $S$ is an $\mcH$-realizable sample, $y=(y_1,\ldots,y_{\dds+1})$ is a vertex in $\mcG(\mcH)$. Let $\sigmak$ denote the orientation that minimizes the maximum $k$-outdegree of $\mcG(\mcH')$ chosen by $\mcA_\mcH$ in Step 2. \Cref{lem:dds+1-outdeg-dds} ensures that the $k$-outdegree of $\sigmak$ is at most $\dds$. Let $e_i$ be the edge in the $i^{\text{th}}$ direction adjacent to $y$. For every $i \in [\dds+1]$, we have $\mu_{S_{-i}}(x_i) \not\owns y_i \ \Leftrightarrow \ \sigmak(e_{i}) \not\owns y$. Hence,
    \begin{align*}
        \sum_{i=1}^{\dds+1} \mathds{1}\left[\mu_{S_{-i}}(x_i) \not\owns y_i \right]=\sum_{i=1}^{\dds+1} \mathds{1}\left[\sigmak(e_{i}) \not\owns y \right]=\outdeg(y;\sigmak)\leq \dds.
    \end{align*}
    It follows that there must exist $i \in [\dds+1]$ such that $\mu_{S_{-i}}(x_i) = y_i$.
\end{proof}

Using \Cref{algo:one-incl} as our base, we are ready to describe a list-learning algorithm which outputs a list having a \textit{large} size, but is nevertheless a successful learner, in the sense described as follows:

\begin{definition}[Weak List PAC Learner]
    An algorithm $\mcA$ with sample size $m$ and list size $k$
    is a weak list PAC learner with success probability $\alpha >0$
    for the concept class $\mcH \subseteq \mcY^\mcX$ if
    for every distribution $\mcD$ on $\mcX \times \mcY$ realizable by $\mcH$,
    \begin{align*}
        \Pr_{(S,(x,y)) \sim \mcD^{m+1}}\left[y \in \mu^k_S(x) \right] \geq \alpha,
    \end{align*}
    where $\mu^k_S = \mcA(S)$ is a $k$-list hypothesis. The list size $k$ can depend on the success probability $\alpha$.
\end{definition}

\begin{remark}
    \label{rem:weak-list-learner-vs-strong}
    While the definition above is identical to list PAC learnability as defined in \cite[Definition 30]{brukhim2022characterization}, as mentioned above in \Cref{rem:list-learning-definition-in-BCD+22}, the list size $k$ can increase arbitrarily with the success probability $\alpha$ in this definition. Since we are aiming for list learnability for a \textit{fixed} list size \textit{irrespective} of the success probability in \Cref{def:list-pac-learnability}, we denote the above definition as ``weak'' list learnability.
\end{remark}

We now give a weak list-learning algorithm for hypothesis classes having finite $k$-DS dimension.

\begin{algorithm}[H]
\caption{Weak list PAC learner $\mcL_{\mcH,t}$ for $\mcH \subseteq \mcY^\mcX$ with $\dds(\mcH)=\dds$ and $t \in \N$.} \label{algo:weak-list-learner}
\begin{flushleft}
  {\bf Input:} Data $S \in (\mcX \times \mcY)^m$ where $m=\dds+t$.  \\
{\bf Output:} A $k'$-list hypothesis $\mu^{k'}_S$ for $k' = k\binom{m}{t}$. \\
\end{flushleft}
\begin{algorithmic}[1]
\STATE Let $S_1,\ldots,S_{\binom{m}{t}}$ denote all subsamples of $S$ of size $\dds$.
\STATE Let $\mu^k_{S_i}=\mcA_{\mcH}(S_i)$ denote the $k$-list hypothesis output by \Cref{algo:one-incl} on input sample $S_i$.
\STATE {\bf Output:} Return the $k'$-list hypothesis $\mu^{k'}_S$ defined by,
\[
\mu^{k'}_S(x) = \bigcup_{i=1}^{\binom{m}{t}}\mu^k_{S_i}(x).
\]
\end{algorithmic}
\end{algorithm}

\begin{proposition}
    \label{prop:weak-list-learner}
    Let $\mcH \subseteq \mcY^\mcX$ be a hypothesis class with $k$-DS dimension $\dds<\infty$ and let $t\in\N$. \Cref{algo:weak-list-learner} is a
    weak list PAC learner for $\mcH$ with sample size $m = \dds+t$, list size $k' = k\binom{m}{t}$ and success probability $\alpha=\frac{t+1}{\dds+t+1}$.
\end{proposition}
\begin{proof}
    Let $\mu^{k'}_S = \mathcal{L}_{\mcH,t}(S)$ be the $k'$-list generated by \Cref{algo:weak-list-learner} on input data $S$. By the leave-one-out symmetrization argument \cite[Fact 14]{brukhim2022characterization}, for any $\mcD$ realizable by $\mcH$,
    \begin{align*}
        \Pr_{(S,(x,y)) \sim \mcD^{m+1}}\left[y \in \mu^{k'}_S(x) \right] &= \Pr_{(S',i) \sim \mcD^{m+1} \times \mathrm{Unif}([m+1])}\left[y_i \in \mu^{k'}_{S'_{-i}}(x_i)\right], 
    \end{align*}
    where $\mathrm{Unif}[m+1]$ denotes the uniform distribution on $[m+1]$. It hence suffices to show that every fixed realizable sample $S'$ of size $m+1$ satisfies
    \begin{equation}
        \Pr_{i \sim \mathrm{Unif}([m+1])}\left[y_i \in \mu^{k'}_{S'_{-i}}(x'_i )\right] \ge \frac{t+1}{\dds+t+1}.
    \end{equation}
    Let us call an index~$i\in [m+1]$ {\it good} if $y_i \in \mu^{k'}_{S'_{-i}}(x'_i)$. We want to show that there are at least $t+1$ good indices. By \Cref{claim:one-incl-gets-one-correct}, at least one of the indices in $[\dds+1]$ is good. Denote this good index by $i_1$. Again, by \Cref{claim:one-incl-gets-one-correct}, at least one of the indices in $[\dds+2] \setminus \{i_1\}$ is good. Denote this good index by $i_2$. Repeat this process to obtain the required $t+1$ good indices.
\end{proof}

\begin{remark}
    \label{rem:weak-already-non-trivial}
    \Cref{algo:weak-list-learner} is already a non-trivial weak list learner \textit{even for classes with infinite DS dimension}. In contrast, the list learning algorithm in \cite[Algorithm 2]{brukhim2022characterization} as is will not be able to weak list learn such a class for any list size. Recall that we already saw an example of a hypothesis class that has finite $k$-DS dimension but infinite DS dimension in \Cref{eg-finite-ds-infinite-k-ds}. Although the algorithm above is a straightforward analogue of the corresponding algorithm in \cite{brukhim2022characterization}, it highlights the need to design algorithms based on the \textit{$k$-DS dimension} instead of the \textit{DS dimension} to list learn a class.
\end{remark}

We now start building up towards how we can bring down the list size to $k$ and come up with a list learning algorithm in the strong sense of \Cref{def:list-pac-learnability}. We will utilize the fact that hypothesis classes with finite $k$-DS dimension (and hence finite $k$-Natarajan dimension) allow for list orientations of their associated one-inclusion graphs with small $k$-outdegree (\Cref{sec:shifting} above). Similar to \cite{brukhim2022characterization}, our final learning algorithm will be the consequence of the existence of a $k$-list sample compression scheme for hypothesis classes with finite $k$-DS dimension.

First, we state the definition of list realizability, and describe an algorithm for list learning under the list realizability assumption.
\begin{definition}[List realizability \cite{brukhim2022characterization}]
A sample $S \in (\mcX \times \mcY)^m$ is \emph{realizable} by the list $\mu$ if $y\in \mu(x)$ for every $(x,y)$ in $S$. A distribution $\mcD$ over $\mcX \times \mcY$ is {\em realizable} by $\mu$ if for every $m \in \N$, a random sample $S\sim \mcD^m$ is realizable by $\mu$ with probability~$1$. 
\end{definition}

\begin{algorithm}[H]
\caption{One-inclusion list algorithm $\mcA_{\mcH,\mu^{k'}}$ 
for a class $\mcH$ and list $\mu^{k'}$} 
\label{algo:list-learning-for-list-realizable-distributions}
\begin{flushleft}
  {\bf Input:} A sample $S = \big((x_1, y_1),\ldots,(x_m, y_m)\big)$ realizable by $\mcH$ and $\mu^{k'}$. \\
{\bf Output:} A $k$-list hypothesis $\mu^k_{S}: \mcX \to \mcY$. \\
\ \\
For each $x \in \mcX$, the $k$-list $\mu^k_S(x)$ is computed as follows.
\end{flushleft}
\begin{algorithmic}[1]
\STATE Consider the class $\mcH' \subseteq \mcY^{m+1}$ of all patterns over the \emph{unlabeled data} that are realizable by both $\mcH$ and $\mu^{k'}$ i.e.,
$\mcH'=\{h \in \mcH|_{(x_1,\ldots,x_m,x)} : h(m+1) \in \mu^{k'}(x) \text{ and } h(i) \in \mu^{k'}(x_i)\}$.
\STATE Find a $k$-list orientation $\sigmak$ of $\mcG(\mcH')$ that \textit{minimizes} the \textit{maximum} $k$-outdegree.
\STATE Consider the edge in direction $m+1$ defined by $S$:
\[e =\{ h \in \mcH':  \forall i \in [m]  \ \ 
h(i) =y_i\}.\]
\STATE Set $\mu^k_S(x) = \{h(m+1):h \in \sigmak(e)\}$.
\end{algorithmic}
\end{algorithm}

\begin{remark}
    \label{rem:algos-identical}
    \Cref{algo:list-learning-for-list-realizable-distributions} is essentially identical to \Cref{algo:weak-list-learner}, except for the fact that it takes into account the list realizability information in Step 1.
\end{remark}

The following proposition shows that if we only care about distributions realizable by a known list $\mu$, $k$-list learning (in the strong sense) is possible. We want to think of this known list as having a size \textit{much larger} than the list size $k$ we care about (otherwise we simply output $\mu$ and the proposition would be trivial). Eventually, we want to be able to $k$-list learn without the list realizability assumption, but this will be an important building block towards that.

\begin{proposition}
    \label{prop:list-learning-for-list-realizable-distributions}
    Let $\mcH \subseteq \mcY^\mcX$ be a class with $k$-Natarajan dimension $\dnat = \dnat(\mcH) <\infty$ and let $\mu^{k'}$ be a $k'$-list (where $k' \gg k$). For every distribution $\mcD$ over $\mcX \times \mcY$ that is realizable by both $\mcH$ and by~$\mu^{k'}$, and for all integers $m>0$,
    \begin{align*}
        \Pr_{(S,(x,y)) \sim \mcD^{m+1}}\left[\mu^k_S(x) \not\owns y \right] \leq \frac{240k^4 \dnat \log(k')}{m},
    \end{align*}
    where $\mu^k_S = \mcA_{\mcH,\mu^{k'}}(S)$.
\end{proposition}

\begin{proof}
    Let $\mcD$ be a distribution that is realizable by $\mcH$ and $\mu^{k'}$. By the leave-one-out symmetrization argument,
    \begin{align*}
        \Pr_{(S, (x,y))\sim \mcD^{m+1} }\left[\mu^k_{S}(x) \not\owns y \right] 
        = \Pr_{(S',i) \sim \mcD^{m+1} \times \mathrm{Unif}(m+1)}\left[ \mu^k_{S'_{-i}}(x'_i) \not\owns y'_i\right],
    \end{align*}
    where $\mu^k_S = \mcA_{\mcH,\mu^{k'}}(S)$. It therefore suffices to show that for every sample $S'$ that is realizable by $\mcH$ and $\mu^{k'}$,
    \begin{equation}
        \Pr_{i \sim \mathrm{Unif}(m+1)}\left[\mu^k_{S'_{-i}}(x'_i) \not\owns y'_i \right] \le \frac{240k^4\dnat\log(k')}{m}.
    \end{equation}
    Fix $S'$ that is realizable by $\mcH$ and $\mu^{k'}$ for the rest of the proof. The class {$\mcH' = \mcH|_{(x'_1,\ldots,x'_{m+1})}$} constructed by the algorithm $\mcA_{\mcH,\mu^{k'}}$ for $S'_{-i}$ and $x'_i$ is the same for all values of~$i$, and is realizable by $\mu^{k'}$. The $k$-Natarajan dimension of $\mcH'$ is at most that of $\mcH$. Denote by $\sigmak$ the orientation of $\mcG(\mcH')$ that the algorithm chooses. \Cref{thm:outdeg-bound-in-terms-of-dnat} tells us that the maximum $k$-outdegree of $\sigmak$ is at most $240k^4\dnat\log(k'))$. Let~$y'$ denote the vertex in $\mcG(\mcH')$ defined by $y'=(y'_1,\ldots,y'_{m+1})$, and let $e_i$ denote the edge in the $i^{\text{th}}$ direction adjacent to $y'$. Then, we have that
    \begin{align*}
        \Pr_{i \sim \mathrm{Unif}(m+1)}\left[\mu^k_{S'_{-i}}(x'_i)  \not\owns y'_i \right] &= \frac{1}{m+1}\sum_{i=1}^{m+1}\mathds{1}\left[\mu^k_{S'_{-i}}(x'_i)  \not\owns y'_i \right] \\
        &= \frac{1}{m+1}\sum_{i=1}^{m+1}\mathds{1}\left[\sigmak(e_i) \not\owns y' \right] \\
        &= \frac{\outdeg(y';\sigmak)}{m+1} \\
        &\le \frac{240k^4\dnat\log(k')}{m+1} \le \frac{240k^4\dnat\log(k')}{m}.
    \end{align*}
\end{proof}

We now have all the necessary ingredients to describe a list sample compression scheme for hypothesis classes with finite $k$-DS dimension.

\subsection{Sample compression scheme}
\label{sec:sample-compression}

We first state some definitions:

\begin{definition}[List sample compression scheme \cite{brukhim2022characterization}]
    \label{def:list-sample-compression-scheme}
    An $m\to r$ list sample compression scheme $(\kappa, \rho)$ with list size~$k$ consists of {a \emph{reconstruction function} 
    \[\rho:(\mcX\times\mcY)^r\to \{Y \subseteq \mcY : |Y| \leq k\}^\mcX\] such} that for every $\mcH$-realizable $S\in (\mcX \times \mcY)^m$, there exists $S'\in (\mcX\times \mcY)^r$ whose elements appear in $S$ (so that $\kappa(S)=S'$) such that for every $(x,y)$ in $S$ we have $y \in \mu^k(x)$, where $\mu^k=\rho(S')$.
\end{definition}

\begin{definition}[List sample compression scheme given list realizability]
    \label{def:list-sample-compression-scheme-given-list-realizability}
    An $m\to r$ list sample compression scheme $(\kappa, \rho)$ with list size $k$ for a class $\mcH$ given list realizability with respect to a list $\mu^{k'}$ ($k' \gg k$) consists of a \emph{reconstruction function} \[\rho:(\mcX\times\mcY)^r\to \{Y \subseteq \mcY : |Y| \leq k\}^\mcX\] such that for every $S\in (\mcX\times \mcY)^m$ that is realizable by both $\mcH$ and $\mu^{k'}$, there exists $S'\in (\mcX\times \mcY)^r$ whose elements appear in $S$ (so that $\kappa(S)=S'$) such that for every $(x,y)$ in $S$ we have
    $y \in \mu^k$, where $\mu^k=\rho(S')$.
\end{definition}

The following two lemmas state that 1) every class with finite $k$-DS dimension is list compressible for some \textit{large} list size $k'$, and 2) every class with finite $k$-Natarajan dimension is list compressible with list size $k$ given list-realizability with respect to some list of size $k' \gg k$. Recall that our ultimate goal is to show that a hypothesis class with finite $k$-DS dimension is list compressible with list size $k$. Step 1) above achieves a weaker version of this goal --- it says that we can compress the class to a list size $k' \gg k$. But in combination with step 2), we finally achieve our goal.\footnote{It would have been ideal if we could \textit{directly} achieve 1) for list size $k$, but our two-step analysis (which mimics that of \cite{brukhim2022characterization}) only allows us to first get the list size down to some large (but bounded) $k' \gg k$, and then to $k$.}

\begin{lemma}
    \label{lem:k'-list-compressible}
    Let $\mcH \subseteq \mcY^\mcX$ be a class with $k$-DS dimension $\dds=\dds(\mcH) < \infty$. For every integers $m,t>0$, there exists an $m\to r_1$ list sample compression scheme $(\kappa_1, \rho_1)$ for $\mcH$ with list size $k'$, where 
    \begin{align*}
    r_1 \leq \frac{\dds+t+1}{t+1} (\dds+t) \log(2m),
    \end{align*}
    and
    \begin{align*}
    k' \leq  k \binom{\dds + t +1}{t+1} \log(2m).
    \end{align*}
\end{lemma}

\begin{proof}
    We first describe the reconstruction function $\rho_1$. Let $l=\frac{\dds+t+1}{t+1}\log(2m)$. Given an $\mcH$-realizable sample $S'$ of size $r_1 = (\dds+t)l$, partition it into $l$ contiguous subsequences $S'_1,\dots,S'_l$ each having size $\dds+t$. Define $\mu^{k'}=\rho_1(S')$ as
    \begin{align*}
        \mu^{k'}(x) = \bigcup_{j=1}^l \mu_j(x),
    \end{align*}
    where $\mu_j=\mcL_{\mcH, t}(S'_j)$ is the size $k\binom{\dds+t}{t}$ size list hypothesis output by \Cref{algo:weak-list-learner} on input sample $S'_j$. This results in $\mu^{k'}$ having size $k'=lk\binom{\dds+t}{t}\le k\binom{\dds + t +1}{t+1} \log(2m)$ as required.

    We now show the existence of $S'$, for which we will use the probabilistic method. Let $\mcD=\mathrm{Unif}(S)$ denote the uniform distribution over the $m$ examples in $S$, and let $\alpha=\frac{t+1}{\dds+t+1}$. Let $S'_1$ be a draw of $\dds+t$ samples from $\mcD$ and let $\mu_1=\mcL_{\mcH, t}(S'_1)$. \Cref{prop:weak-list-learner} with this choice of $\mcD$ implies that
    \begin{align*}
        \Pr_{(S'_1,(x,y)) \sim \mcD^{\dds+t+1}}[y \in \mu_1(x)] &= \E_{S'_1 \sim \mcD^{\dds+t}}\E_{(x,y)\sim\mcD}[[\mathds{1}[y \in \mu_1(x)]] \\
        &= \E_{S'_1 \sim \mcD^{\dds+t}} \left[ \frac{1}{m}\sum_{i=1}^m \mathds{1}[y_i \in \mu_1(x_i)]\right] \\
        &\ge \alpha.
    \end{align*}
    In particular, this means that there exists $S'_1$ which satisfies $\sum_{i=1}^m\mathds{1}[y_i \in \mu_1(x_i)] \ge \alpha m$. Remove from $S$ all the $\ge \alpha m$ examples $(x_i, y_i)$ that satisfy $y_i \in \mu_1(x_i)$, and repeat the same reasoning as above on the remaining sample. In this way, at each step $j$, we find a sample $S'_j$ and list $\mu_j=\mcL_{\mcH, t}(S'_j)$ that covers at least an $\alpha$-fraction of the remaining examples in $S$. After $l$ steps, all the examples in $S$ are covered because $(1-\alpha)^lm \le \frac{1}{2m} < 1$. Setting $S'$ to be the concatenation of $S'_1,\dots,S'_l$ completes the proof.
\end{proof}

\begin{lemma}
    \label{lem:k-list-compressible-given-list-realizability}
    Let $\mcH \subseteq \mcY^\mcX$ be a class with $k$-Natarajan dimension $\dnat<\infty$ and let $\mu^{k'}$ be a $k'$-list. For every integer $m>0$, there exists an $m\to r_2$ list sample compression scheme $(\kappa_2,\rho_2)$ with list size $k$ for $\mcH$ given list realizability with respect to $\mu^{k'}$ that satisfies 
    \begin{align*}
    r_2 \leq  11520k^6 \dnat \log(k')\log (2m).
    \end{align*}
\end{lemma}
\begin{proof}
    We first describe the reconstruction function $\rho_2$. Let $l=12k\log(2m)$, $n= 960k^5\dnat\log(k')$. Given a sequence $S'$ of $r_2=nl$ examples that are realizable by both $\mcH$ and $\mu^{k'}$, partition it into $l$ contiguous subsequences $S_1',\ldots,S'_l$ each of size $n$. Define $\mu^k(x)=\rho_2(S')$ as
    \begin{equation}
        \label{eqn:top-k-rule}
        \mu^{k}(x) = \topk(\mu^k_1(x),\ldots,\mu^k_l(x)),
    \end{equation}
    where $\mu^k_j=\mcA_{\mcH, \mu^{k'}}(S'_j)$ is the $k$-list output by \Cref{algo:list-learning-for-list-realizable-distributions} on input sample $S'_j$, and $\topk(y_1,\ldots,y_l)$ is a list of the $k$ most frequently occurring labels among the $l$ lists (ties broken arbitrarily).

    We now show the existence of $S'$ for every sample $S$ realizable by $\mcH$ and $\mu^{k'}$, again by using the probabilistic method. We will instantiate von Neumann's minimax theorem \cite{v1928theorie} for this.

    We first claim that there exists a distribution $\mcP$ over sequences $T$ of size $n$ with elements from $S$ such that for every fixed example $(x,y) \in S$,
    \begin{equation}
        \label{eqn:mixed-strategy-minnie}
        \Pr_{T \sim \mcP}[y \notin \mu^{k}_T(x)] \le \frac{1}{2(k+1)},
    \end{equation}
    where $\mu_T = \mcA_{\mcH, \mu^{k'}}(T)$. Consider a zero-sum game between two players Max and Minnie, where Max's pure strategies are examples $(x,y)\in S$, and Minnie's pure strategies are sequences $T \in S^n$. The cost matrix $L$ is defined by $L_{(x,y),T}=\mathds{1}[y \notin \mu^k_T(x)]$. Let $\mcQ$ be any mixed strategy by Max i.e., $\mcQ$ is a distribution over the $m$ examples in $S$. The distribution $\mcQ$ is realizable by both $\mcH$ and $\mu^{k'}$. \Cref{prop:list-learning-for-list-realizable-distributions} with $\mcD = \mcQ$ implies that
    \begin{align*}
        \Pr_{(T, (x,y))\sim\mcQ^{n+1}}[y \notin \mu^k_T(x)] &\le \frac{240k^4\dnat\log(k')}{n} \\
        &=\frac{240k^4\dnat\log(k')}{960k^5\dnat\log(k')} = \frac{1}{4k} \le \frac{1}{2(k+1)}.
    \end{align*}
    Observe that the LHS above is the evaluation of the cost corresponding to mixed strategies $\mcQ$ and $\mcQ^n$ by Max and Minnie respectively. In other words, for every mixed strategy $\mcQ$ by Max, there exists a mixed strategy by Minnie (namely that of playing $\mcQ^n$) for which the cost is at most $\frac{1}{2(k+1)}$. By von Neumann's minmax theorem, this implies that there exists a mixed strategy by Minnie such that for every mixed (and in particular, pure) strategy by Max, the cost is at most $\frac{1}{2(k+1)}$. This mixed strategy is the required $\mcP$ that satisfies \cref{eqn:mixed-strategy-minnie}. 

    We can now argue the existence of the required $S'$. Let $S'_1,\ldots,S'_l$ be i.i.d. samples of size $n$ from $\mcP$. Using \cref{eqn:mixed-strategy-minnie} and a Chernoff bound, for every fixed $(x,y) \in S$, we have that
    \begin{align*}
        \Pr_{(S'_1,\ldots,S'_l) \sim \mcP^l}\left[\frac{1}{l}\sum_{j=1}^l \mathds{1}[y \notin \mu^k_{S'_j}(x)] \ge \frac{1}{k+1}\right] \le \exp\left(-\frac{l}{6(k+1)}\right) \le \exp\left(-\frac{l}{12k}\right) < \frac{1}{m}.
    \end{align*}
    By a union bound over the $m$ examples in $S$, we get that with positive probability, it holds that $\frac{1}{l}\sum_{j=1}^l \mathds{1}[y \in \mu^k_{S'_j}(x) > 1-\frac{1}{k+1}$ for all $(x,y)\in S$ simultaneously. In particular, there exist $S'_1,\ldots,S'_l$ for which this guarantee holds. But this means that for every $(x,y) \in S$, $y$ appears in the $k$ most frequently occurring labels amongst the $l$ lists for that $x$. This is because the guarantee implies that $y$ occurs in strictly more than $\frac{kl}{k+1}$ lists. Assume for the sake of contradiction that $y$ is not among the $k$ most frequently occurring labels. Then there must exist $k$ labels other than $y$ each occurring in strictly more than $\frac{kl}{k+1}$ lists. But then the amount of room remaining for $y$ in the concatenation of the lists is strictly less than $kl - k\left(\frac{kl}{k+1}\right) = \frac{kl}{k+1}$, which we know is not true. Hence, $y$ belongs to the list of $k$ most frequently occurring labels. Thus, we have shown that there exist $S'_1,\ldots,S'_l$ such that $y\in \mu^k(x)$ for all $(x,y) \in S$, where $\mu_k$ is the $\topk$ rule defined in \cref{eqn:top-k-rule}. Setting $S'$ to the concatenation of $S'_1,\ldots, S'_l$ completes the proof.
\end{proof}

\Cref{lem:k'-list-compressible} and \Cref{lem:k-list-compressible-given-list-realizability} equip us to derive the following theorem.

\begin{theorem}[Finite $k$-DS dimension implies $k$-list compressibility]
    \label{thm:finite-dds-implies-compressibility}
    Let $\mcH \subseteq \mcY^\mcX$ be a hypothesis class with $k$-DS dimension $\dds=\dds(\mcH) < \infty$ and $k$-Natarajan dimension $\dnat=\dnat(\mcH)$. For every integers $m,t>0$, there exists an $m \to r$ list sample compression scheme $(\kappa, \rho)$ for $\mcH$ with list size $k$ satisfying
    \begin{equation}
    \label{eqn:final-compression-size}
    r \le \left(\frac{\dds+t+1}{t+1}(\dds+t) + 11520k^6\dnat\log\left(k\binom{\dds+t+1}{t+1}\log(2m)\right)\right)\log(2m).
    \end{equation}
\end{theorem}
\begin{proof}
    Let $S$ be an $\mcH$-realizable sample of size $m$. \Cref{lem:k'-list-compressible} tells us that there is a reconstruction $\rho_1$ that produces $k'$-lists, and a sequence $S'$ of $r_1$ examples from $S$ such that $S$ is realizable with respect to $\mu^{k'}=\rho_1(S)$. We feed this $\mu^{k'}$ into \Cref{lem:k-list-compressible-given-list-realizability}. The guarantee in \Cref{lem:k-list-compressible-given-list-realizability} applied to $\mcH$ and $\mu^{k'}$ implies that there is a reconstruction $\rho_2$, and a sequence $S''$ of $r_2$ examples from $S$ such that $\mu^k=\rho_2(S'')$ realizes the entire sample $S$. The composition of the two schemes gives us the required $m \to r_1 + r_2$ $k$-list compression scheme for $\mcH$, with $\rho_2$ as the overall reconstruction function.
\end{proof}

\begin{remark}
    \label{rem:asymptotics-sample-compression-params}
    We want to think of the parameters in \cref{eqn:final-compression-size} in the regime where $m \to \infty$, and $\dds, t, k$ are constants. In this regime, we have that
    \begin{align*}
        r = \widetilde{O}\left(\frac{(\dds+t)^2}{t} + tk^6\dds\right),
    \end{align*}
    where the $\widetilde{O}(\cdot)$ is hiding $\polylog(m,\dds,k)$ factors. Thus, in this regime, setting $t = \sqrt{\dds}$ gives us
    \begin{equation}
        \label{eqn:final-compression-size-asymptotics}
        r = \widetilde{O}(k^6(\dds)^{1.5}).
    \end{equation}
\end{remark}

\subsection{List compression implies list learnability}
\label{sec:compression-implies-learnability}

\cite{david2016statistical} establish an equivalence between PAC learnability (realizable and agnostic) and sample compression. By simply interpreting the zero-one loss as $\mathds{1}[y \neq \mu^k(x)]$ instead of $\mathds{1}[y \neq h(x)]$, the proofs for their results translate almost identically to the list learning setting, establishing an equivalence between list PAC learnibility and list sample compression. For completeness, we provide slightly more elementary versions of their results in the list learning setting in \Cref{sec:compression-implies-learnability-appendix}.

In particular, since \Cref{thm:finite-dds-implies-compressibility} guarantees the existence of a list compression scheme $(\kappa, \rho)$ for any $\mcH$ having finite $k$-DS dimension, \Cref{cor:list-compression-implies-learnability-realizable} in \Cref{sec:list-compression-implies-learnability-realizable} then implies that for $\mcA^{real}=(\kappa, \rho)$, for any distribution $\mcD$ realizable by $\mcH$,
\begin{align*}
    L_{\mcD}(\mcA^{real}(S)) \le O\left(\frac{r\log(m) + \log(1/\delta)}{m}\right).
\end{align*}
Additionally, \Cref{thm:list-compression-implies-learnability-agnostic} in \Cref{sec:list-compression-implies-learnability-agnostic} also implies the existence of a list learning rule $\mcA^{agn}$ satisfying
\begin{align*}
    L_{\mcD}(\mcA^{agn}(S)) \le \inf_{h \in \mcH}L_{\mcD}(h) + O\left(\sqrt{\frac{r\log(m)+\log(1/\delta)}{m}}\right).
\end{align*}
Substituting $r$ from \cref{eqn:final-compression-size-asymptotics} above completes the proof of \Cref{thm:sufficiency}.

\section*{Acknowledgements}
The authors are supported by Moses Charikar's Simons Investigator Award. We thank Amit Daniely for bringing to our notice their prior application of the list-variant of Sauer's lemma in \cite{daniely2015inapproximability}.

\bibliographystyle{alpha}
\bibliography{references}

\newpage
\appendix
\section{List orientations for infinite graphs}
\label{sec:orientations-for-infinite-graphs-appendix}
We complete the proof of \Cref{lem:dds+1-outdeg-dds} in the case where $\mcH \subseteq \mcY^{d+1}$ has infinite size.

Let $\mcG(\mcH)=(V,E)$ be the one-inclusion graph of $\mcH$. Let $\mcZ$ be the set of pairs $\mcZ=\{(v,e) \in V \times E:v \in e\}$. For $z=(v,e) \in \mcZ$, define the discrete topology on the set $X_{z}=\{0,1\}$ with $\tau_z = 2^{X_z}=\{\emptyset, \{0\}, \{1\}, \{0,1\}\}$ as the open sets. By Tychonoff's theorem, $\mcK = \prod_{z \in \mcZ}X_{z} = \{0,1\}^{\mcZ}$ is a compact set with respect to the product topology.

We can think of any list orientation $\sigmak$ of $\mcG(\mcH)$ of size $k$ as an element $\kappa \in \mcK$. Concretely, if $v \in \sigmak(e)$, we set $\kappa_{(v,e)}=1$, and conversely, if $v \notin \sigmak(e)$, we set $\kappa_{(v,e)}=0$.

We will define some good sets. For every $v \in V$, let $A_v \subseteq \mcK$ be the set 
\begin{equation}
    \label{eqn:A_v}
    A_v=\{\kappa \in \mcK: \kappa_{(v,e)}=1 \text{ for at least one } e \ni v \}.
\end{equation}
Further, for every $v \in V$ and $j \in [d+1]$, let $e_j$ be the edge adjacent to $v$ in the $j^{\text{th}}$ direction. Let $B_{v,j} \subseteq \mcK$ be the set 
\begin{equation}
    \label{eqn:B_v_j}
    B_{v,j}=\{\kappa \in \mcK: \kappa_{(u,e_j)}=1 \text{ for at most } k \text{ vertices } u \in V\}.
\end{equation}
Now, the complement of the set $A_v$ is $\bar{A}_v=\{\kappa \in \mcK: \forall i \in [d+1], \kappa_{(v, e_i)}=0\}$, where $e_i$ is the edge adjacent to $v$ in the $i^{\text{th}}$ direction. This set is open in the product topology, because we can write it as $\prod_{z=(u,e):u \neq v}\{0,1\} \times \prod_{z=(v,e_i):i \in [d+1]}\{0\}$ (i.e., it is a Cartesian product of open sets in the underlying topologies where only a finite number are not equal to all of $X_z=\{0,1\}$). Similarly, the complement of the set $B_{v,j}$ is 
\begin{align*}
     \bar{B}_{v,j}=\bigcup_{u_1\neq\dots\neq u_{k+1} \in V}\{\kappa \in \mcK: \kappa_{(u_i,e_j)}=1\; \forall i \in [k+1]\},
\end{align*}
which is again an open set in the product topology. To see this, fix some $u_1\neq\dots\neq u_{k+1} \in V$, and then we can again write $\{\kappa \in \mcK: \kappa_{(u_i,e_j)}=1\; \forall i \in [k+1]\}$ as the Cartesian product of open sets in the underlying topologies where only a finite number are not equal to all of $\{0,1\}$. This implies that each such set is open in the product, and since open sets are closed under arbitrary unions, the union over $u_i$s is an open set too. Thus, each $A_v$ and $B_{v,j}$ is a closed set, and hence the set $\Sigma_v = A_v \cap \bigcap_{j \in [d+1]}B_{v,j}$ is also closed, by virtue of being a finite intersection of closed sets.

We will now argue that for every finite $U \subset V$, the set $\bigcap_{v \in U}\Sigma_v$ is non-empty. For the finite hypergraph $\mcG_U$ that $\mcG(\mcH)$ induces on $U$, by the inductive proof in \Cref{lem:dds+1-outdeg-dds} above, there exists a list orientation $\sigmak$ with maximum $k$-outdegree at most $d$. We can interpret $\sigmak$ as defining an element $\kappa$ in $\bigcap_{v \in U}\Sigma_v$. Let us think of all the coordinates in $\kappa$: these correspond to edges in $\mcG(\mcH)$. Some of these edges have counterparts in $\mcG_U$ --- we set the coordinates in $\kappa$ corresponding to these edges to 0/1 precisely according to the list of vertices they have been oriented to in $\sigmak$. For all the edges that ``disappeared'' in $\mcG_U$, we set these coordinates to $0/1$ arbitrarily. We can check that this $\kappa$ is a valid member of $\bigcap_{v \in U}\Sigma_v$, as it satisfies the conditions that define the sets $A_v$ and $B_{v,j}$.

What we have just shown is that the collection $\{\Sigma_v\}_{v \in V}$ of closed sets in the product topology satisfies the \textit{finite intersection property}. Due to compactness, this means that $\bigcap_{v \in V}\Sigma_v \neq \emptyset$. Let $\kappa^* \in \bigcap_{v \in V}\Sigma_v$. We can now interpret $\kappa^*$ as a partial list orientation. The reason it might be partial is that some edges might not be oriented at all (i.e., $\kappa^*$ might be 0 at all the coordinates corresponding to these edges). But since $\kappa^*$ lies in the intersection of $\Sigma_v$ over all $v \in V$, by definition of these sets, surely it must be the case that every vertex is adjacent to at least one edge that is list-oriented towards it. Also, every edge (if oriented in $\kappa^*$ at all) is oriented towards at most $k$ vertices. Thus, the maximum outdegree of the partial orientation given by $\kappa^*$ is at most $d$. We can complete $\kappa^*$ to a full orientation by arbitrarily orienting all unoriented edges to a list of $\le k$ vertices they are adjacent to. This does not increase the out-degree of any vertex, and we are done.

\section{List compression implies list learnability}
\label{sec:compression-implies-learnability-appendix}

Let $\Theta$ denote the set of all list functions $\mu^k : \mcX \to \{Y \subseteq \mcY:|Y|\le k\}$. Define the $0-1$ loss function $l:\Theta \times (\mcX \times \mcY) \to \R_+$ as
\begin{equation}
    \label{eqn:0-1-loss}
    l(\mu^k, (x,y)) = \mathds{1}[y \notin \mu^k(x)].
\end{equation}
The \textit{risk} of a list function $\mu^k$ with respect to a distribution $\mcD$ over $\mcX \times \mcY$ is then defined as
\begin{equation}
    \label{eqn:risk}
    L_{\mcD}(\mu^k) = \E_{(x,y)\sim \mcD}[l(\mu^k, (x,y))].
\end{equation}
The empirical risk of a list function $\mu^k$ with respect to a sample $S = \{(x_1,y_1),\ldots,(x_m, y_m)\} \in (\mcX \times \mcY)^m$ is defined as
\begin{equation}
    \label{eqn:empirical-risk}
    L_S(\mu^k) = \frac{1}{m}\sum_{i=1}^m l(\mu^k, (x_i, y_i)).
\end{equation}

\subsection{Bernstein's inequality for a generalization bound}
\label{sec:bernstein}

We can instantiate Bernstein's inequality to obtain a generalization bound for a list function $\mu^k$. The following is a rewrite of \cite[Lemma B.10]{shalev2014understanding}, but in the context of the risks for list functions defined above.

\begin{lemma}
    \label{lem:bernstein}
    Fix an arbitrary list function $\mu^k$ and let $\mcD$ be an arbitrary distribution over $\mcX \times \mcY$. Then, for any $\delta \in (0,1)$, we have
    \begin{align*}
    &1.\qquad \Pr_{S \sim \mcD^m}\left[L_S(\mu^k) \ge L_D(\mu^k) + \sqrt{\frac{2L_D(\mu^k)\log(1/\delta)}{m}} + \frac{2\log(1/\delta)}{3m}\right] \le \delta. \\
    &2.\qquad \Pr_{S \sim \mcD^m}\left[L_D(\mu^k) \ge L_S(\mu^k) + \sqrt{\frac{2L_S(\mu^k)\log(1/\delta)}{m}} + \frac{4\log(1/\delta)}{m}\right] \le \delta.
    \end{align*}
\end{lemma}
\begin{proof}
    Define the random variables $\alpha_1,\ldots,\alpha_m$ such that $\alpha_i=l(\mu^k, (x_i, y_i))-L_D(\mu^k)$. Observe that $\E[\alpha_i]=0, \; \forall i \in [m]$. We have
    \begin{align*}
        \E[\alpha_i^2] &= \E[l(\mu^k, (x_i, y_i))^2] - 2L_{\mcD}(\mu^k)\E[l(\mu^k, (x_i, y_i))] + L_{\mcD}(\mu^k)^2 \\
        &= \E[l(\mu^k, (x_i, y_i))]- L_{\mcD}(\mu^k)^2 \qquad (\text{since $l(\mu^k, (x_i, y_i))^2 = l(\mu^k, (x_i, y_i))$}) \\
        &\le \E[l(\mu^k, (x_i, y_i))] = L_{\mcD}[\mu^k].
    \end{align*}
    Applying Bernstein's inequality to the $\alpha_i$s yields
    \begin{align*}
        \Pr\left[\sum_{i=1}^m \alpha_i > t\right] \le \exp\left(-\frac{t^2/2}{mL_{\mcD}(\mu^k)+t/3}\right) = \delta.
    \end{align*}
    Solving for $t$ gives
    \begin{align*}
        &\frac{t^2/2}{mL_{\mcD}(\mu^k)+t/3} = \log(1/\delta) \\
        \implies\qquad & t^2/2-\frac{\log(1/\delta)}{3}t-\log(1/\delta)mL_{\mcD}(\mu^k)=0 \\
        \implies\qquad & t = \frac{\log(1/\delta)}{3} + \sqrt{\frac{\log^2(1/\delta)}{9}+2\log(1/\delta)mL_{\mcD}(\mu^k)} \\
        &\le \frac{2\log(1/\delta)}{3} + \sqrt{2\log(1/\delta)mL_{\mcD}(\mu^k)} .
    \end{align*}
    Thus, we have shown that for
    \begin{align*}
    &\Pr_{S \sim \mcD^m}\left[ mL_S(\mu^k) - mL_D(\mu^k) > \frac{2\log(1/\delta)}{3} + \sqrt{2\log(1/\delta)mL_{\mcD}(\mu^k)}\right] \le \delta \\
    \implies\qquad&\Pr_{S \sim \mcD^m}\left[L_S(\mu^k) \ge L_D(\mu^k) + \sqrt{\frac{2L_D(\mu^k)\log(1/\delta)}{m}} + \frac{2\log(1/\delta)}{3m}\right] \le \delta,
    \end{align*}
    which is the first part of the lemma. For the second part, define $\beta_1,\ldots,\beta_m$ such that $\beta_i=L_D(\mu^k)-l(\mu^k, (x_i, y_i))$. Observe that $\E[\beta_i]=0, \; \forall i \in [m]$. An identical calculation as above gives $\E[\beta_2^2] \le L_{\mcD}(\mu^k)$, and applying Bernstein's inequality, we obtain that with probability at least $1-\delta$,
    \begin{align*}
        L_D(\mu^k) \le L_S(\mu^k) + \sqrt{\frac{2L_D(\mu^k)\log(1/\delta)}{m}} + \frac{2\log(1/\delta)}{3m}.
    \end{align*}
    We can restate the above as
    \begin{align*}
        &x-b\sqrt{x}-c \le 0 \\
        \text{where} \qquad &x = L_D(\mu^k) \\
        &b = \sqrt{\frac{2\log(1/\delta)}{m}} \\
        &c = L_S(\mu^k) + \frac{2\log(1/\delta)}{3m}.
    \end{align*}
    \cite[Lemma 19, Appendix A.5]{shalev2007online} then implies that
    \begin{align*}
        &x \le c + b^2 + b\sqrt{c} \\
        \implies \qquad &L_D(\mu^k) \le L_S(\mu^k)+ \frac{2\log(1/\delta)}{3m} + \frac{2\log(1/\delta)}{m} + \sqrt{\frac{2\log(1/\delta)}{m}\left(L_S(\mu^k)+\frac{2\log(1/\delta)}{3m}\right)} \\
        \implies \qquad &L_D(\mu^k) \le L_S(\mu^k)+ \frac{2\log(1/\delta)}{3m} + \frac{2\log(1/\delta)}{m} + \sqrt{\frac{2L_S(\mu^k)\log(1/\delta)}{m}} + \frac{2\log(1/\delta)}{\sqrt{3}m} \\
        \implies \qquad &L_D(\mu^k) \le L_S(\mu^k) + \sqrt{\frac{2L_S(\mu^k)\log(1/\delta)}{m}} + \frac{4\log(1/\delta)}{m}.
    \end{align*}
\end{proof}

\subsection{List compression implies learnability in the realizable case}
\label{sec:list-compression-implies-learnability-realizable}

\begin{theorem}
    \label{thm:list-compression-implies-learnability-realizable}
    Let $r$ be an integer and let $B:(\mcX \times \mcY)^r \to \{Y \subseteq \mcY: |Y|\le k\}^\mcX$ be a mapping from a sequence of $r$ examples to a $k$-list hypothesis. Let $m \ge 2r$, and let $\mcA:(\mcX \times \mcY)^m \to \{Y \subseteq \mcY: |Y|\le k\}^\mcX$ be a list learning rule that receives a sequence $S$ of size $m$ as input and returns the list hypothesis $\mcA(S)=B((x_{i_1}, y_{i_1}),\ldots,(x_{i_r}, y_{i_r}))$ for some $(i_1,\ldots,i_r) \in [m]^r$. Let $V = \{(x_j, y_j): j \notin (i_1,\ldots,i_r)\}$ be the sequence of examples which were not selected in determining $\mcA(S)$. Then with probability at least $1-\delta$ over the choice of $S$ drawn from any distribution $\mcD$, we have
    \begin{align*}
        L_{\mcD}(\mcA(S)) \le L_V(\mcA(S)) + \sqrt{L_V(\mcA(S))\frac{4\log(m^r/\delta)}{m}} + \frac{8\log(m^r/\delta)}{m}.
    \end{align*}
\end{theorem}
\begin{proof}
    Say we draw a sample $S=\{(x_1, y_1), \ldots, (x_m, y_m)\}$ of size $m$. For a fixed sequence $I = (i_1,\ldots,i_r) \in [m]^k$, let $T_I=\{(x_{j}, y_j): j \in I\}$ and let $V_I=\{(x_j, y_j): j \notin I$\}. Let $\mu^k_{T_I}=B(T_I)$. Since $T_I$ and $V_I$ are independent, \Cref{lem:bernstein} implies that
    \begin{align*}
        &\Pr_{S \sim \mcD^m}\left[L_D(\mu^k_{T_I}) \ge L_{V_I}(\mu^k_{T_I}) + \sqrt{\frac{2L_{V_I}(\mu^k_{T_I})\log(1/\delta')}{|V_I|}} + \frac{4\log(1/\delta')}{|V_I|}\right] \le \delta' \\
        \implies \qquad&\Pr_{S \sim \mcD^m}\left[L_D(\mu^k_{T_I}) \ge L_{V_I}(\mu^k_{T_I}) + \sqrt{\frac{2L_{V_I}(\mu^k_{T_I})\log(1/\delta')}{m-r}} + \frac{4\log(1/\delta')}{m-r}\right] \le \delta'.
    \end{align*}
    Note that
    \begin{align*}
        &L_{\mcD}(\mcA(S)) \ge L_V(\mcA(S)) + \sqrt{L_V(\mcA(S))\frac{2\log(1/\delta')}{m-r}} + \frac{4\log(1/\delta')}{m-r} \\
        \implies \qquad & \exists I \in [m]^r \text{ s.t. } L_D(\mu^k_{T_I}) \ge L_{V_I}(\mu^k_{T_I}) + \sqrt{\frac{2L_{V_I}(\mu^k_{T_I})\log(1/\delta')}{m-r}} + \frac{4\log(1/\delta')}{m-r}.
    \end{align*}
    Therefore, applying a union bound over all the possible choices of $I$, we get
    \begin{align*}
        &\Pr_{S \sim \mcD^m}\left[L_{\mcD}(\mcA(S)) \ge L_V(\mcA(S)) + \sqrt{L_V(\mcA(S))\frac{2\log(1/\delta')}{m-r}} + \frac{4\log(1/\delta')}{m-r}\right] \\
        &\le \Pr_{S \sim \mcD^m}\left[\exists I \in [m]^r \text{ s.t. } L_D(\mu^k_{T_I}) \ge L_{V_I}(\mu^k_{T_I}) + \sqrt{\frac{2L_{V_I}(\mu^k_{T_I})\log(1/\delta')}{m-r}} + \frac{4\log(1/\delta')}{m-r}\right] \\
        &\le \sum_{I \in [m]^r}\Pr_{S \sim \mcD^m}\left[L_D(\mu^k_{T_I}) \ge L_{V_I}(\mu^k_{T_I}) + \sqrt{\frac{2L_{V_I}(\mu^k_{T_I})\log(1/\delta')}{m-r}} + \frac{4\log(1/\delta')}{m-r}\right] \\
        &\le m^r\delta'.
    \end{align*}
    Setting $\delta' = \frac{\delta}{m^r}$, and under the assumption that $r \le m/2 \implies m-r\ge m/2$, we get
    \begin{align*}
        &\Pr_{S \sim \mcD^m}\left[L_{\mcD}(\mcA(S)) \ge L_{V}(\mcA(S)) + \sqrt{L_{V}(\mcA(S))\frac{4\log(m^r/\delta)}{m}} + \frac{8\log(m^r/\delta)}{m}\right] \le \delta \\
    \end{align*}
\end{proof}

If $\mcA$ is a valid list sample compression scheme with list size $k$, we obtain the following corollary:
\begin{corollary}
    \label{cor:list-compression-implies-learnability-realizable}
    Let $(\kappa, \rho)$ be an $m \to r$ list sample compression scheme with list size $k$, such that $\kappa$ is the selection function, $\rho$ is the reconstruction function and $r \le m/2$. For any $\mcH$-realizable distribution $\mcD$, we have that with probability at least $1-\delta$ over the choice of $S \sim \mcD^m$,
    \begin{align*}
        L_{\mcD}(\rho(\kappa(S))) \le \frac{8r\log(m)+8\log(1/\delta)}{m}.
    \end{align*}
\end{corollary}
\begin{proof}
    Since $\rho(\kappa(S))$ correctly (list)labels \textit{all of} $S$ by virtue of being a valid compression scheme, $L_V(\rho(\kappa(S)))=0$ in particular, and the result follows.
\end{proof}

\subsection{List compression implies learnability in the agnostic case}
\label{sec:list-compression-implies-learnability-agnostic}

First, we define what an agnostic list sample compression means.

\begin{definition}[Agnostic list sample compression scheme]
    \label{def:agnostic-list-sample-compression}
    An $m \to r$ agnostic list sample compression scheme with list size $k$ consists of a selection scheme $\kappa$ and a reconstruction function $\rho$
    \begin{align*}
        \rho: (\mcX \times \mcY)^r \to \{Y \subseteq \mcY : |Y| \le k\}^{\mcX}
    \end{align*}
    such that for every sample $S \in (\mcX \times \mcY)^m$ (not necessarily realizable by $\mcH$), $\kappa(S)=S'$ for some $S' \in S^r$, and 
    \begin{align*}
        L_{S}(\rho(\kappa(S))) \le \inf_{h \in \mcH}L_S(h).
    \end{align*}
\end{definition}

\begin{lemma}[List sample compression $\implies$ agnostic list sample compression]
    \label{lem:list-compression-implies-agnostic-list-compression}
    If $\mcH$ has an $m \to r$ list sample compression scheme with list size $k$, then $\mcH$ has an $m \to r$ agnostic list sample compression scheme with list size $k$.
\end{lemma}
\begin{proof}
    Let $(\kappa, \rho)$ be a list sample compression scheme for $\mcH$ with list size $k$. Given an arbitrary sample $S$ of size $m$, let $h^* \in \mcH$ be an arbitrary minimizer of $L_S(h)$ (since we are using the $0-1$ loss function, the infimum is attained). Let $\tilde{S}$ denote the sub-sample of $S$ correctly labelled by $h^*$. Then, $\tilde{S}$ is realizable by $\mcH$. Since $(\kappa, \rho)$ is a valid list compression scheme for $\mcH$, we have that $L_{\tilde{S}}(\rho(\kappa(\tilde{S})))=0$. Thus, applying the compression scheme on the sub-sample $\tilde{S}$ yields a list hypothesis $\tilde{\mu}^k=\rho(\kappa(\tilde{S}))$ which correctly (list)labels all of $\tilde{S}$. Since $h^*$ incorrectly labels \textit{all} the remaining examples in $S$ other than $\tilde{S}$, the loss of $\tilde{\mu}^k$ on these examples can be no worse. Hence, $L(\tilde{\mu}^k)\le \inf_{h \in \mcH}L_S(h)$.
\end{proof}

\begin{theorem}[List compression implies agnostic list PAC learnability]
    \label{thm:list-compression-implies-learnability-agnostic}
    Let $(\kappa, \rho)$ be an $m \to r$ $k$-list sample compression scheme for $\mcH$ with compression size $r \le m/2$. Then, there exists a $k$-list learning algorithm $\mcA$ such that for any distribution $\mcD$ on $(\mcX \times \mcY)$, we have that with probability at least $1-\delta$ over the choice of $S \sim \mcD^m$,
    \begin{align*}
        L_{\mcD}(\mcA(S)) \le \inf_{h \in \mcH}L_{\mcD}(h) + O\left(\sqrt{\frac{r\log(m)+\log(1/\delta)}{m}}\right).
    \end{align*}
\end{theorem}
\begin{proof}
    From \Cref{lem:list-compression-implies-agnostic-list-compression}, the existence of $(\kappa, \rho)$ implies the existence an \textit{agnostic} list compression scheme $(\tilde{\kappa}, \tilde{\rho})$ for $\mcH$  with the same compression size $r$. Let $\mcA(S) = \tilde{\rho}(\tilde{\kappa}(S))$. Observe that (for any $\mcA$)
    \begin{align*}
        L_V(\mcA(S)) &\le L_S(\mcA(S)) + \frac{r}{m},
    \end{align*}
    and hence \Cref{thm:list-compression-implies-learnability-realizable} implies that with probability at least $1-\frac{\delta}{2}$ over $S \sim \mcD^m$,
    \begin{align}
        L_{\mcD}(\mcA(S)) &\le L_S(\mcA(S)) + \frac{r}{m} + \sqrt{L_{V}(\mcA(S))\frac{4\log(2m^r/\delta)}{m}} + \frac{8\log(2m^r/\delta)}{m} \nonumber \\
        &\le L_S(\mcA(S)) + \sqrt{\frac{4\log(2m^r/\delta)}{m}} + \frac{8\log(2m^r/\delta)+r}{m}. \label{eqn:ub2-agnostic}
    \end{align}
    Since $\mcA(S)$ is an agnostic list sample compression scheme, by definition, we have
    \begin{equation}
        \label{eqn:agnostic}
        L_S(\mcA(S)) \le \inf_{h \in \mcH} L_{S}(h).
    \end{equation}
    Further, since $S \sim \mcD^m$ i.i.d., for every $h \in \mcH$, $L_S(h)$ is of the form $\frac{1}{m}\sum_{i=1}^m Z_i$, where $Z_i$s are i.i.d. random variables taking values in $\{0,1\}$, each having expectation $L_{\mcD}(h)$. Hoeffding's inequality implies
    \begin{equation}
        \label{eqn:hoeffding}
        \Pr_{S \sim \mcD^m}\left[L_S(h) \ge L_{\mcD}(h) + \sqrt{\frac{\log(1/\delta)}{2m}}\right] \le \delta.
    \end{equation}
    By the defnition of infimum, $\exists h' \in \mcH$ such that
    \begin{equation}
        \label{eqn:infimum}
        L_{\mcD}(h') \le \inf_{h \in \mcH}L_{\mcD}(h) + \left(\sqrt{\frac{\log(1/\delta)}{m}} - \sqrt{\frac{\log(1/\delta)}{2m}}\right).
    \end{equation}
    Thus, we can say that
    \begin{align*}
        \Pr_{S \sim \mcD^m}\left[L_S(\mcA(S)) \ge \inf_{h}L_{\mcD}(h) + \sqrt{\frac{\log(1/\delta)}{m}}\right] &\le \Pr_{S \sim \mcD^m}\left[\inf_{h}L_S(h) \ge \inf_{h}L_{\mcD}(h) + \sqrt{\frac{\log(1/\delta)}{m}}\right] \\
        &\hspace{5cm} (\text{from \cref{eqn:agnostic}}) \\
        &\le \Pr_{S \sim \mcD^m}\left[L_S(h') \ge \inf_{h}L_{\mcD}(h) + \sqrt{\frac{\log(1/\delta)}{m}}\right] \\
        &\hspace{4cm} (\text{definition of infimum}) \\
        &\le \Pr_{S \sim \mcD^m}\left[L_S(h') \ge L_{\mcD}(h') + \sqrt{\frac{\log(1/\delta)}{2m}}\right] \\
        &\hspace{5cm} (\text{from \cref{eqn:infimum}}) \\
        &\le \delta. \\
        &\hspace{5cm} (\text{from \cref{eqn:hoeffding}})
    \end{align*}
    Thus, we have that with probability at least $1-\frac{\delta}{2}$ over $S \sim \mcD^m$, 
    \begin{equation}
        \label{eqn:ub1-agnostic}
        L_S(\mcA(S)) \le \inf_{h \in \mcH}L_{\mcD}(h) + \sqrt{\frac{\log(2/\delta)}{2m}}.
    \end{equation}
    Applying a union bound to \cref{eqn:ub1-agnostic} and \cref{eqn:ub2-agnostic}, we get that with probability at least $1-\delta$ over $S \sim \mcD^m$,
    \begin{align*}
        L_{\mcD}(\mcA(S)) &\le \inf_{h \in \mcH}L_{\mcD}(h) + \sqrt{\frac{\log(2/\delta)}{2m}} + \sqrt{\frac{4\log(2m^r/\delta)}{m}} + \frac{8\log(2m^r/\delta)+r}{m} \\
        &= \inf_{h \in \mcH}L_{\mcD}(h) + \sqrt{\frac{\log(2/\delta)}{2m}} + \sqrt{\frac{4\log2 + 4r\log m + 4\log(1/\delta)}{m}} \\
        &\qquad\qquad\qquad\qquad+\frac{8\log2 + 8r\log m + 8\log(1/\delta)+r}{m}\\
        &= \inf_{h \in \mcH}L_{\mcD}(h) + O\left(\sqrt{\frac{r\log(m)+\log(1/\delta)}{m}}\right),
    \end{align*}
    as required.
\end{proof}

\end{document}